\documentclass[final,5p,times,twocolumn]{elsarticle}

\usepackage{amsmath}
\usepackage{color}
\usepackage{bm}
\usepackage{amsfonts}
\usepackage{amsthm}
\usepackage{algorithm}
\usepackage{algorithmic}
\usepackage{amssymb}
\usepackage{tikz}
\usepackage{bbm}
\usepackage[
  locale = DE
]{siunitx}
\usepackage[version = 3]{mhchem}
\usepackage{lipsum}
\usepackage{booktabs}
\usepackage{array}
\usepackage{textcomp}
\usepackage{multirow} 
\usepackage{multicol}

\usepackage{subcaption}
\usepackage[normalem]{ulem}

\usepackage{url}
\usepackage{graphicx}
\usepackage{microtype}

\newtheorem{lemma}{Lemma}
\newtheorem{theorem}{Theorem}

\newtheorem{proposition}{Proposition}

\newtheorem{remark}{Remark}

\DeclareMathOperator*{\tr}{tr}

\journal{Neural Networks}

\begin{document}

\begin{frontmatter}

\title{Optimized Weight Initialization on the Stiefel Manifold\\ for Deep ReLU Neural Networks}

\author[{label1,label2}]{Hyungu Lee}
\author[{label2}]{Taehyeong Kim}
\author[{label1,label2}]{Hayoung Choi\corref{cor1}}
\cortext[cor1]{Corresponding author.}
\ead{hayoung.choi@knu.ac.kr}

\address[label1]{Department of Mathematics, Kyungpook National University, Daegu 41566, Republic of Korea.}
\address[label2]{Nonlinear Dynamics and Mathematical Application Center, Kyungpook National University, Daegu 41566, Republic of Korea.}


\begin{abstract}
Stable and efficient training of ReLU networks with large depth is highly sensitive to weight initialization.
Improper initialization can cause permanent neuron inactivation (``dying ReLU'') and exacerbate gradient instability as network depth increases.
Methods such as He, Xavier, and orthogonal initialization preserve variance or promote approximate isometry.
However, they do not necessarily regulate the pre-activation mean or control activation sparsity, and their effectiveness often diminishes in very deep architectures.
This work introduces a semi-orthogonal initialization specifically optimized for ReLU by solving an optimization problem on the Stiefel manifold, thereby preserving scale and calibrating the pre-activation statistics from the outset.
A family of closed-form solutions and an efficient sampling scheme are derived.
Theoretical analysis at initialization shows that the proposed initialization mitigates dying ReLU, slows the decay of activation variance, and alleviates gradient vanishing, thereby stabilizing signal and gradient flow in deep architectures.
Empirically, across MNIST, Fashion-MNIST, multiple tabular datasets, few-shot settings, and ReLU-family activations, the proposed initialization maintains stable performance as network depth increases, with particularly clear advantages in few-shot learning and large-depth settings.
\end{abstract}

\begin{keyword}
Weight initialization \sep Semi-orthogonal matrix \sep Deep learning \sep FFNN \sep ReLU activation function
\end{keyword}

\end{frontmatter}

\section{Introduction}

Deep learning has achieved remarkable success across various domains, driven by its ability to learn effective representations from data through deep neural networks~\cite{hanin2021principles}.
These architectures come in various forms, each designed to address specific data processing challenges.
A representative example is the feed-forward neural network (FFNN), which serves as a fundamental building block in deep learning, where information propagates unidirectionally from the input layer, through multiple hidden layers, to the output layer~\cite{goodfellow2016deep}.
The depth of the network determines representational capacity, since additional layers enable the progressive extraction of more abstract and complex features~\cite{bengio2009learning,lecun2015deep}.
On the other hand, increasing depth may lead to the vanishing and exploding gradient phenomenon that hinders practical training and delays convergence~\cite{glorot2010understanding}.
The Rectified Linear Unit (ReLU) activation function is widely adopted in modern deep learning models due to its simple form, computational efficiency, and sparsity-inducing properties in neuron activity~\cite{goodfellow2016deep,glorot2010understanding,nair2010rectified}.
However, the ReLU activation suffers from the dying ReLU problem, 
when a pre-activation remains negative, its gradient vanishes and the corresponding parameters stop updating~\cite{maas2013rectifier,he2015delving}.
This problem is closely related to vanishing gradients~\cite{glorot2010understanding}, and it becomes more pronounced in deeper networks.
To address the dying ReLU problem, various approaches have been proposed.
The main approaches can be broadly classified into three categories.
First, some studies employ activation functions that permit small gradients for negative inputs, such as Leaky ReLU~\cite{maas2013rectifier} or the exponential linear unit (ELU)~\cite{clevert2015fast}, or modify the network structure itself by introducing residual connections~\cite{he2016deep}.
Second, there are widely used methods that apply normalization techniques, such as batch normalization~\cite{ioffe2015batch}, to each layer to stabilize the distribution of activation values.
Lastly, there are studies on initialization techniques that carefully set weights and biases to prevent neuron deactivation early in training~\cite{glorot2010understanding,he2015delving}.
Among these three approaches, this paper focuses on weight initialization, which is involved in the most fundamental stage of training.
It proposes a new method to enhance the stability and performance of ReLU networks.
Early methods, such as Xavier initialization~\cite{glorot2010understanding}, were designed under the assumption of symmetric, zero-centered activations (e.g., tanh, sigmoid), and often underperform with ReLU-based models due to the asymmetry of the function.
In response, He initialization~\cite{he2015delving} scales weight variances by $2/m$, where $m$ is the number of input units (fan-in) of the layer to compensate for ReLU’s half-activation.
More recent approaches exploit geometric constraints and dynamical-isometry conditions to better tailor initialization to ReLU’s properties~\cite{pennington2017resurrecting, saxe2013exact}.
In a recent study, Lee et al.~\cite{lee2024improved} proposed a specific constructive method for orthogonal weight matrices by applying QR decomposition to an all-ones matrix perturbed by a small term $\epsilon>0$.
While it is effective, this constructive method treats a key property—approximate angle preservation with the all-ones vector—as a byproduct of its procedure rather than a direct optimization objective.
Consequently, these desirable characteristics are only approximately guaranteed.

To address these limitations, this paper introduces a novel weight initialization method derived from a principled optimization framework.
We formulate the design of the weight matrix as an optimization problem on the Stiefel manifold, which seeks a semi-orthogonal matrix that is maximally aligned with the all-ones vector.
This approach is motivated by our empirical and theoretical findings that such alignment is critical for preventing neuron inactivation in ReLU networks.
The resulting initialization scheme, obtained as the exact solution to this optimization problem, is designed to inherently preserve signal propagation and stabilize training dynamics from the very first epoch, thereby mitigating the dying ReLU and vanishing gradient problems in deep and narrow architectures.
To validate the practical benefits of our approach, the proposed initialization is evaluated on several benchmark tasks.
Deep ReLU neural networks trained on MNIST and Fashion-MNIST are evaluated in terms of test accuracy after a fixed training budget~\cite{lecun1998mnist}.
Robustness to activation choice is assessed through experiments with several ReLU-family activations.
In few-shot learning scenarios—where only a handful of labeled examples are available—the initialization maintains stability and accuracy.
Finally, the application to diverse tabular datasets demonstrates its versatility across different data modalities.
Across these settings, the proposed scheme exhibits strong depth stability and particularly clear gains in few-shot learning, while remaining competitive across ReLU-family activations and tabular tasks.
The main contributions of this paper are summarized as follows.
\begin{enumerate}
    \item We solve an optimization problem on the Stiefel manifold to derive a novel weight initialization and analyze its key mathematical properties.
    \item We design an efficient algorithm that enables practical weight initialization.
    \item We theoretically and empirically demonstrate that the proposed initialization alleviates early-stage issues in ReLU networks such as dying neurons and unstable gradients.
    \item We demonstrate stable performance as network depth increases and particularly clear advantages in few-shot learning, while evaluating robustness across ReLU-family activations and tabular tasks.
\end{enumerate}

\subsection*{Notations}

Throughout this paper, we adopt the following notation conventions. 
Let \(\mathbb{R}\) (resp. \(\mathbb{R}_{+}\)) be the set of real numbers (resp. nonnegative real numbers).
The standard inner product of two vectors \(\bm{u}\) and \(\bm{v}\) is denoted by \(\langle \bm{u}, \bm{v} \rangle\), and \(\|\bm{v}\|_2\) denotes the Euclidean norm.
The vector max norm is denoted by \(\|\bm{x}\|_\infty\).
The superscript $T$ denotes the transpose operator.
Denote the all-ones vector of size \(m\) by \(\bm{1}_m = [1,1,\dots,1]^T\in\mathbb{R}^m\) and its normalized version by \(\bm{\xi}_m = \tfrac{1}{\sqrt m}\bm{1}_m\).
The all-ones matrix of size \(m\times n\) is denoted by \(\bm{J}_{m\times n}\) where the subscript is shortened to \(\bm{J}_m\) for a square matrix of size \(m\).
The \(m\times m\) identity matrix is denoted by \(\bm{I}_m\).
Denote the zero matrix of size \(m\times n\) by \(\bm{0}_{m \times n}\).
When the dimension is clear from context, we omit subscripts on \(\bm{1}\), \(\bm{\xi}\), \(\bm{J}\), \(\bm{I}\), and \(\bm{0}\).
We use the notation \(x \sim \mathcal{N}(\mu, \sigma^2)\) to indicate that the scalar random variable \( x \) follows a univariate normal distribution with mean \( \mu \in \mathbb{R} \) and variance \( \sigma^2 > 0 \).
For the multivariate case, $\bm{x} \sim \mathcal{N}(\boldsymbol{\mu}, \Sigma)$
denotes that the random vector \( \bm{x} \in \mathbb{R}^d \) follows a multivariate normal distribution with mean vector \( \bm{\mu} \in \mathbb{R}^d \) and covariance matrix \( \bm{\Sigma} \in \mathbb{R}^{d \times d} \), which is symmetric and positive semidefinite.
We write \(x \sim \mathcal{U}(a, b)\) to denote that \( x \) is distributed uniformly over the interval \( [a, b] \subset \mathbb{R} \), where \( a < b \).

\section{Preliminaries}\label{sec:prior_work}

Before introducing our proposed weight initialization method, we provide a brief overview of the basic concepts and review prior works.
\subsection{Basic Concepts}
FFNN is a composition of affine maps and pointwise nonlinearities.
Given an input \(\bm{x}^0\in\mathbb{R}^{N_0}\), the forward recursion for layers \(\ell=1,\ldots,L\) is given by
\begin{equation}\label{eq:ffnn}
    \bm{y}^\ell = \bm{W}^\ell \bm{x}^{\ell-1} + \bm{b}^\ell, \qquad
    \bm{x}^\ell = \phi(\bm{y}^\ell),
\end{equation}
where \(N_{\ell}\) denotes the number of units in the $\ell$-th layer.
Here, \(\bm{W}^\ell \in \mathbb{R}^{N_\ell \times N_{\ell-1}}\) is the weight matrix, \(\bm{b}^\ell \in \mathbb{R}^{N_\ell}\) is its bias vector, \(\phi: \mathbb{R} \to \mathbb{R}\) is the activation function.
For FFNN architectures based on fully connected layers, non-expanding configurations are widely employed, where the number of neurons either remains constant or decreases across layers.
Such architectures, commonly found in standard designs~\cite{goodfellow2016deep,bengio2006greedy}, support hierarchical compression of features and often lead to improved generalization and interpretability.
Accordingly, in this work, a non-expanding architecture is considered, specified by the condition \(N_0 \geq N_1 \geq \cdots \geq N_L\).
An activation function $\phi$ is a key component in neural networks, introducing non-linearity that allows the model to learn complex, non-linear mappings.
Common choices include sigmoid~\cite{rumelhart1986learning}, hyperbolic tangent (tanh)~\cite{lecun1989backpropagation}, and rectified linear unit (ReLU) functions~\cite{nair2010rectified}, each with distinct characteristics in terms of saturation, gradient flow, and computational cost.
In this paper, the activation function $\phi$ is chosen to be the ReLU, defined element-wise as
\[
\operatorname{ReLU}(\bm{z}) = \max(0, \bm{z}).
\]
The ReLU activation is widely used due to its computational simplicity and beneficial properties for optimization.
It introduces non-linearity while preserving gradient flow for positive inputs, thereby mitigating the vanishing gradient problem.
Furthermore, ReLU induces sparsity in the activations, as it outputs zero for all non-positive values.
While ReLU activations offer several optimization advantages, their effectiveness is closely tied to the choice of weight initialization~\cite{lu2019dying, woorobust}.
Inappropriate initialization may lead to unbalanced activation statistics across layers, undermining training dynamics~\cite{skorski2021revisiting}.
Consequently, considerable research has been devoted to developing initialization schemes, as reviewed below.

\subsection{Prior Works}
Before the widespread adoption of ReLU, activation functions such as sigmoid and tanh were more commonly used.
One of the most influential methods in this category is \emph{Xavier initialization}~\cite{glorot2010understanding}, which draws a weight matrix $\bm{W}=[W_{ij}] \in \mathbb{R}^{m \times n}$ from a uniform distribution, i.e., for all $i,j$
\[
W_{ij} \sim \mathcal{U}\Bigl(-\sqrt{\tfrac{6}{m+n}}, \,\sqrt{\tfrac{6}{m+n}}\Bigr).
\]
Xavier initialization works well for activation functions with symmetry around zero, effectively preserving variance in shallow networks.
However, because ReLU activations are asymmetric, this method often struggles to avert the vanishing gradient problem in deeper ReLU-based networks.
To address the reduction in activation caused by ReLU zeroing out approximately half of its inputs, He et al.~\cite{he2015delving} proposed to increase the variance more aggressively:
\[
W_{ij} \sim \mathcal{N}\Bigl(0, \tfrac{2}{n}\Bigr),
\]
where weights are sampled from a zero-mean normal distribution with variance \( \frac{2}{n} \).
This strategy compensates for the sparsity introduced by ReLU, helping to preserve the variance of activations in deeper networks.
It has become standard in many ReLU-based architectures, including various residual network designs, although it does not entirely prevent gradient degradation in extremely deep architectures.
Other researchers investigated matrix structures and their impact on signal propagation.
For instance, Saxe et al.~\cite{saxe2013exact} advocated \emph{orthogonal initialization}, leveraging the property $\bm{W} \bm{W}^{T} = \bm{I}$ to preserve signal norms through forward and backward passes in linear regimes.
Orthogonal initialization has been shown to maintain dynamical isometry, thereby stabilizing gradient flow in deep networks.
Building upon this intuition, Hu et al.~\cite{hu2020provable} provided a theoretical justification for orthogonal initialization in deep linear networks.
Their analysis demonstrates that orthogonal weights enable depth-independent convergence, in contrast to Gaussian initialization, which requires network width to increase linearly with depth for efficient training.
This result offers a provable advantage for orthogonal initialization in specific settings.
Recent studies have continued to investigate orthogonality-based initialization under more specialized architectures and training regimes \cite{esguerra2023sparsity,nowak2024sparser}.
Despite these theoretical benefits, orthogonal initialization remains challenging to apply in practice.
Generating exact orthogonal matrices is computationally expensive, particularly for high-dimensional or convolutional layers.
Moreover, its effectiveness diminishes in highly nonlinear regimes, where the orthogonality of weight matrices does not guarantee stable signal propagation due to the nonlinear distortions introduced by activation functions.
To further stabilize signal propagation without relying on standard random distributions, deterministic approaches such as ZerO initialization \cite{zhao2022zero} and IDInit \cite{pan2025idinit} have been developed, exploiting identity or structurally optimized matrices for deep architectures.
Recent work has also examined the sensitivity of training trajectories to small perturbations of initialization~\cite{altintacs2025butterfly}.
While these methods have advanced the field, \emph{deep and narrow} ReLU networks pose additional challenges.
In such architectures, the vanishing gradient and the dying ReLU become especially pronounced.
Recently, Lee et al.~\cite{lee2024improved} addressed these difficulties by proposing a deterministic initialization method that constructs an orthogonal weight matrix via a perturbed all-ones matrix $\bm{1}\bm{1}^T + \epsilon \bm{I}$ for sufficiently small $\epsilon > 0$.
This method preserves orthogonality in a manner particularly advantageous for narrow networks.
Empirically, this weight initialization maintains more positive entries than either He initialization or standard orthogonal matrices, thereby reducing the risk of the dying ReLU.
Furthermore, Lee et al. establish that their construction approximately preserves the angle between any input vector $\bm{x}$ and the all-ones vector $\bm{1}$.
Specifically, the proposed initialization $\bm{W}$ satisfies
\begin{equation}\label{eq:prior}
\frac{\langle \bm{x},\,\bm{1}\rangle}{\|\bm{x}\|\|\bm{1}\|}
\approx
\frac{\langle \bm{W} \bm{x},\,\bm{1}\rangle}{\|\bm{W} \bm{x}\|\|\bm{1}\|},
\end{equation}
so the angle between any positive input $\bm{x}$ and the all-ones vector $\bm{1}$ is preserved by $\bm{W}$.
Since
$\langle \bm{W} \bm{x},\,\bm{1}\rangle
= \langle \bm{x},\,\bm{W}^T\bm{1}\rangle,$

\begin{equation}\label{eq:prior2}
\frac{\langle \bm{x},\,\bm{1}\rangle}{\|\bm{x}\|\|\bm{1}\|}
\approx
\frac{\langle \bm{x},\,\bm{W}^T\bm{1}\rangle}{\|\bm{x}\|\|\bm{W}^T\bm{1}\|},
\end{equation}

the alignment of $\bm{W}^T\bm{1}$ with the input direction determines whether the transformed signal retains a positive mean component.
This observation is particularly relevant for ReLU networks.
Unlike symmetric activations, ReLU is sensitive not only to the scale of pre-activations but also to their sign, since negative pre-activations are truncated to zero and may lead to inactive neurons.
In this context, the all-ones vector is not arbitrary: it represents the direction in which all coordinates move together positively, that is, the shared mean component across coordinates.
Hence, preserving alignment with the all-ones direction favors pre-activations with more positive entries, which mitigates the dying ReLU phenomenon.

\section{Methodology}\label{sec:methodology}
Motivated by these observations, we regard the alignment between the weight matrix and the all-ones vector $\bm{1}$ as a key design principle for controlling ReLU activation.
Denote the set of all semi-orthogonal matrices as
\begin{equation}\label{eq:Omn}
\mathcal{O}_{m,n} := \left\{ \bm{W} \in \mathbb{R}^{m \times n} : \bm{W} \bm{W}^{T} = \bm{I}_m \text{ or } \bm{W}^{T} \bm{W} = \bm{I}_n\right\}.
\end{equation}
If \(m \le n\), the rows of \(\bm{W}\) are orthonormal, so that
\[
\bm{W}\bm{W}^{T} = \bm{I}_{m},
\]
whereas if \(m \ge n\), the columns of \(\bm{W}\) are orthonormal, so that
\[
\bm{W}^{T}\bm{W} = \bm{I}_{n}.
\]

\subsection{Empirical Motivation via Alignment Optimization}

\begin{figure*}[!ht]
  \centering
  \begin{subfigure}[t]{0.48\textwidth}
    \centering
    \includegraphics[width=\linewidth]{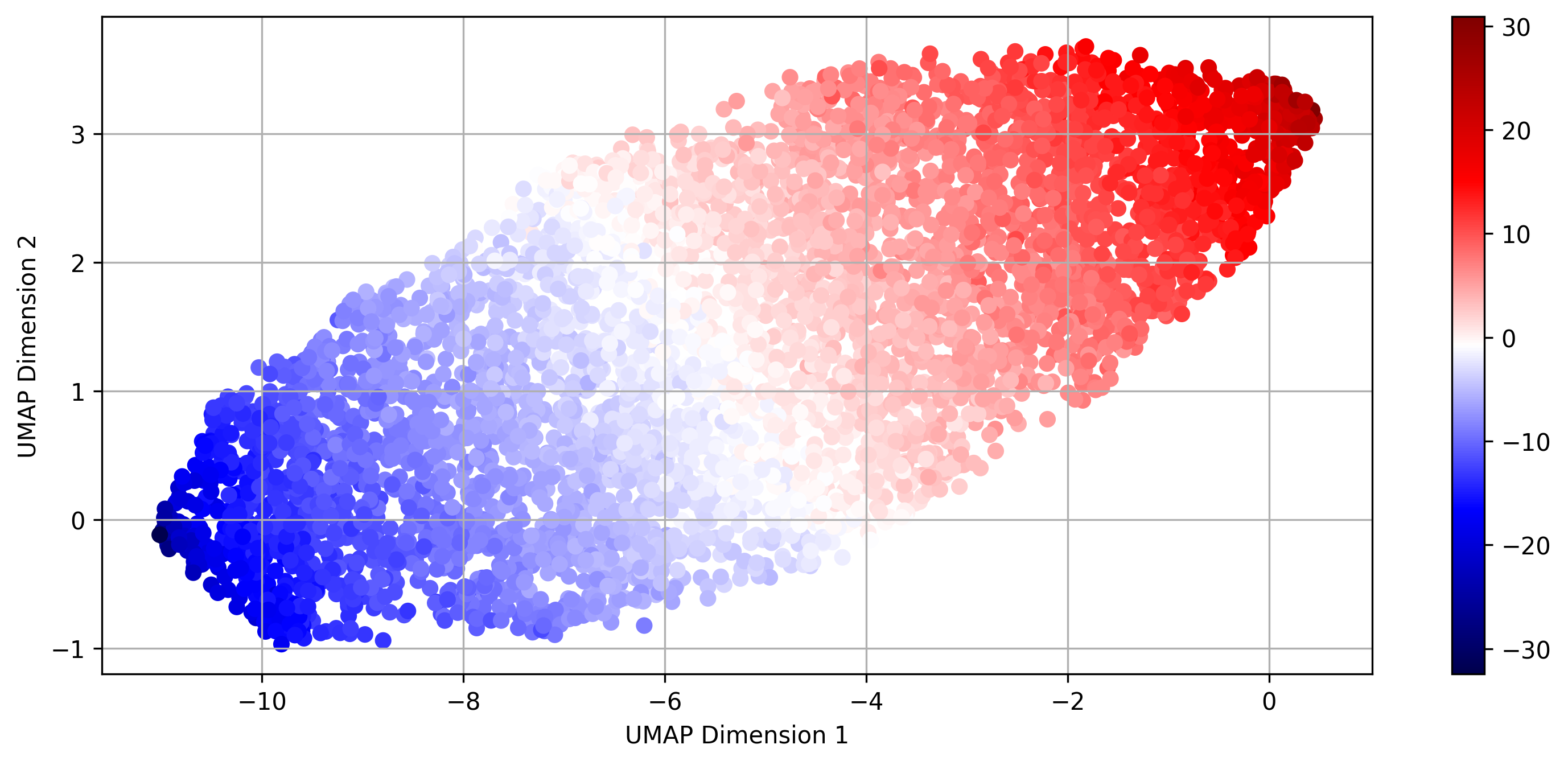}
    \caption{UMAP projection colored by mean inner product value $s_k$}
    \label{fig:image1}
  \end{subfigure}
  \hfill
  \begin{subfigure}[t]{0.48\textwidth}
    \centering
    \includegraphics[width=\linewidth]{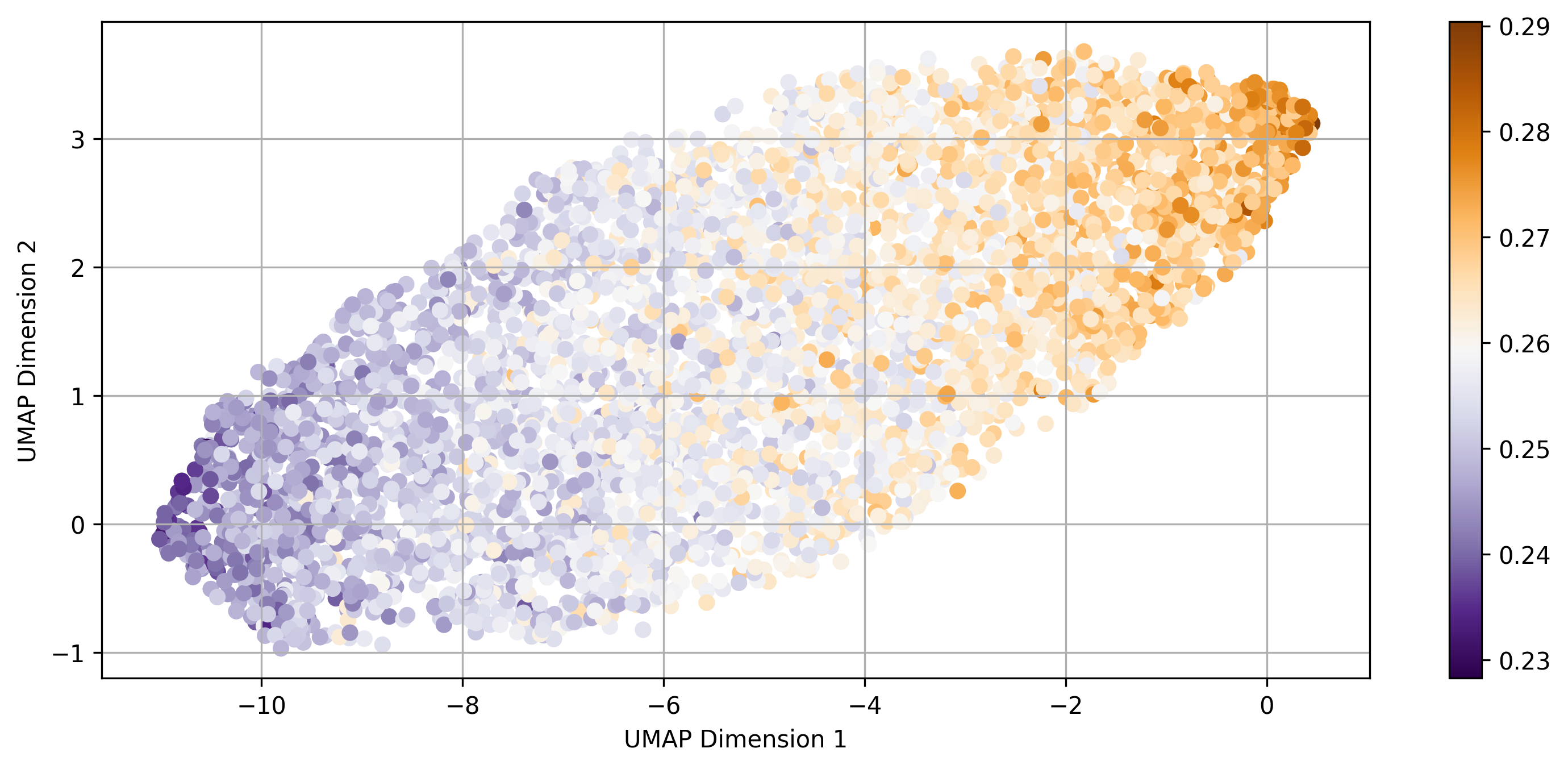} 
    \caption{UMAP projection colored by Classification accuracy $\alpha_k$}
    \label{fig:image2}
  \end{subfigure}
  \caption{A two-dimensional UMAP projection of vectors $\bm v_k = \bm W_{(k)}^{T} \bm{1}$.
}
  \label{fig:fig1}
\end{figure*}

We introduce an optimization-based framework for constructing semi-orthogonal weight matrices that are specifically aligned with directions favorable to ReLU activation.
The guiding principle stems from empirical observations indicating that the alignment between the semi-orthogonal weight matrix and the all-ones vector, representing uniform positive inputs, correlates positively with network performance.
To understand how semi-orthogonal weight matrices align with the all-ones vector, we first observe how randomly sampled semi-orthogonal matrices align with the all-ones vector to understand their effect on ReLU activation behavior.
Consider a collection of random semi-orthogonal matrices,
\[
  \left\{\bm{W}_{(k)}\right\}_{k=1}^{5000} \subset \mathcal{O}_{64,64},
\]
where each $\bm{W}_{(k)}$ is randomly generated using the orthogonal initialization scheme of Saxe et al.~\cite{saxe2013exact}, involving QR decomposition of a Gaussian matrix.
For each $k$, define the 64-dimensional vector $\bm{v}_k = \bm{W}_{(k)}^T \bm{1}$.
Since $\{\bm{v}_k\}_{k=1}^{5000}$ lie in a 64-dimensional space, it is challenging to visualize clustering or distribution patterns directly.
There are several dimension reduction methods to visualize them in a low-dimensional space.
Among them, we adopt Uniform Manifold Approximation and Projection (UMAP)~\cite{mcinnes2018umap} 
because it supports user-defined distance metrics, which allows us to employ the 1-Wasserstein distance.
This choice provides invariance to coordinate permutations and enables a meaningful comparison of the alignment vectors~\cite{peyre2019computational}.
In Fig.~\ref{fig:fig1} (a) and (b), an identical 2D UMAP embedding is used for both plots, i.e., each point occupies the same location in both subfigures and corresponds to the same $\bm{W}_{(k)}$.
In Fig.~\ref{fig:fig1} (a), points are colored by their alignment score
\[
s_k = \langle \bm{v}_k, \bm{1} \rangle=  \bm{1}^T \bm{W}_{(k)} \bm{1} = \sum_{i=1}^{64} \sum_{j=1}^{64} (\bm{W}_{(k)})_{ij},
\]
which quantifies how strongly \( \bm{v}_k \) aligns with the uniform direction \( \bm{1} = (1,1,\dots,1) \);
that is,
\(s_k  =\langle \bm{v}_k, \bm{1} \rangle \)
measures the extent to which \( \bm{v}_k \) projects onto the subspace spanned by \( \bm{1} \), or equivalently, the sum of all entries in \( \bm{W}_{(k)} \).
This reflects how uniformly the components of \( \bm{v}_k \) are distributed across all coordinates, indicating a globally consistent activation pattern.
In Fig.~\ref{fig:fig1} (b), the points are colored instead by the average one-shot classification accuracy $\alpha_k$, obtained when $\bm{W}_{(k)}$ is used to initialize the middle layer of a simple fully connected ReLU network trained on MNIST~\cite{lecun1998mnist}.
Specifically, the network architecture is $784 \to 64 \to 64 \to 10$.
The input-to-hidden ($64\times 784$) and hidden-to-output ($10\times 64$) weight matrices are initialized using orthogonal initialization.
The middle $64 \times 64$ weight matrix is set to the specific sample $\bm{W}_{(k)}$ under evaluation.
For each $\bm{W}_{(k)}$, we run 100 independent one-shot trials: in each trial, one random image per digit class is sampled from the MNIST training set, training proceeds for 100 epochs using vanilla SGD (learning rate 0.01, batch size 64, no weight decay) under cross-entropy loss, and we measure test accuracy on the full MNIST test set.
The color in Fig.~\ref{fig:fig1}(b) represents the average test accuracy $\alpha_k$ across these 100 trials.
Although Fig.~\ref{fig:fig1} (a) and (b) are colored by two different scalar metrics—alignment score $s_k$ and average classification accuracy $\alpha_k$, respectively—the resulting spatial patterns are remarkably similar.
This visual correspondence suggests a positive relationship between the alignment of $\bm{W}_{(k)}$ with the direction of the all-ones vector and the effectiveness of the corresponding initialization.
A Pearson correlation analysis on the dataset $\left\{(s_k, \alpha_k)\right\}_{k=1}^{5000}$ revealed a strong, statistically significant linear correlation between $s_k$ and $\alpha_k$ ($r = 0.8178$).
This finding supports the hypothesis that alignment with the all-ones direction positively contributes to performance in ReLU networks.
Motivated by this result, we construct a new class of weight matrices \( \bm{W} \in \mathbb{R}^{m \times n}\) that satisfy the following two criteria:
\begin{itemize}
    \item[(i)] $\bm{W}$ is semi-orthogonal, 
    \item[(ii)] The value of \( \bm{1}_m^T \bm{W} \bm{1}_n \) is maximized.
\end{itemize} 
Since no more than \( n \) orthonormal vectors can exist in \( \mathbb{R}^n \), the requirement of row orthonormality implies \( m \leq n \).
Throughout the paper, we adopt the assumption \( m \leq n \), i.e., $\bm{W}\bm{W}^T=\bm{I}$.

\subsection{Optimization on the Stiefel Manifold}
Note that $\mathcal{O}_{m,n}$ in \eqref{eq:Omn} is a real Stiefel manifold, which mathematically represents the collection of all ordered orthonormal \( m \)-frames in \( \mathbb{R}^n \).
Each element in \( \mathcal{O}_{m,n} \) can be viewed as a point on a nonlinear manifold embedded in Euclidean space, constrained by \( \frac{1}{2}m(m+1) \) nonlinear equations due to the orthonormality condition.
Notably, \(  \mathcal{O}_{m,n} \) is a smooth Riemannian manifold, and it generalizes both the unit sphere \( \mathbb{S}^{n-1} = \mathcal{O}_{1,n} \) and the orthogonal group \( \mathcal{O}_{n,n}\).
An optimization problem is now formulated over the Stiefel manifold \( \mathcal{O}_{m,n} \) to satisfy the above criteria.
The scalar \( \bm{1}_m^T \bm{W} \bm{1}_n \), representing the sum of all entries of \( \bm{W} \), is equivalent to \( \tr(\bm{J}_{m \times n}^T \bm{W}) \), where \( \bm{J}_{m \times n} \) is the all-ones matrix.
Thus, the optimization can be posed as
\begin{equation}\label{eq:main_optimiz1}
    \max_{\bm{W}\in \mathcal{O}_{m,n}} \tr{(\bm{J}_{m\times n}^T \bm{W})}.
\end{equation}
Equivalently, it can be expressed as
\begin{equation}\label{eq:main_optimiz2}
    \min_{\bm{W}\in \mathcal{O}_{m,n}}  \|\bm{J}_{m\times n} -\bm{W}\|_F,
\end{equation}
which interprets the problem as finding the closest point on the Stiefel manifold to \(\bm{J}_{m\times n}\) in terms of the Frobenius norm.
\begin{theorem}\label{thm:equivalence}
    Every optimal solution of \eqref{eq:main_optimiz2} can be explicitly expressed as follows:
    \begin{equation}
        \widetilde{\bm{W}} = \bm{U}\bm{V}^T,
    \end{equation}
    where $\bm{U} \in \mathcal{O}_{m,m}$ has its first column equal to $\bm{\xi}_{m}$, and $\bm{V} \in \mathcal{O}_{n,m}$ has its first column equal to $\bm{\xi}_{n}$.
\end{theorem}
\begin{proof}
Note that the thin SVD of $\bm{J}_{m\times n}$ is
\[
\bm{J}_{m\times n} = \widehat{\bm{U}}\,\bm{\Sigma}\,\widehat{\bm{V}}^T,
\]
where $\bm{\Sigma}$ is the $m\times m$ diagonal matrix with $\sqrt{mn}$ in the first diagonal element and $0$ elsewhere, $\widehat{\bm{U}} \in \mathcal{O}_{m,m}$ and $\widehat{\bm{V}} \in \mathcal{O}_{n,m}$ have their first columns equal to $\bm{\xi}_{m}$ and $\bm{\xi}_{n}$, respectively.
By the orthogonal invariance of the Frobenius norm, it holds
\[
\|\bm{J}_{m\times n}-\bm{W}\|_F
=\bigl\|\widehat{\bm{U}}^T(\bm{J}_{m\times n}-\bm{W})\,\widehat{\bm{V}}\bigr\|_F
=\bigl\|\bm{\Sigma}-\widehat{\bm{U}}^T\bm{W}\,\widehat{\bm{V}}\bigr\|_F.
\]
By setting $\bm{Q}=\widehat{\bm{U}}^T\bm{W}\,\widehat{\bm{V}}\in\mathcal{O}_{m,m}$, the problem \eqref{eq:main_optimiz2} can be simply rewritten as 
\begin{equation}\label{eq:sigma}
\min_{\bm{Q}\in\mathcal{O}_{m,m}}\|\bm{\Sigma}-\bm{Q}\|_F^2.
\end{equation}
Note that 
\[
\|\bm{\Sigma}-\bm{Q}\|_F^2 = mn + m - 2\sqrt{mn}\,Q_{11},
\]
where $Q_{11}$ is $(1,1)$ entry of $\bm{Q}$.
So, solving \eqref{eq:sigma} is equivalent to maximizing $Q_{11}$. 
Since the first column vector of $\bm{Q}$ has norm $1$, it follows that the maximum possible value of $Q_{11}$ is $1$.
Thus, the optimal solution of \eqref{eq:sigma} is
\[
\widetilde{\bm{Q}} = \begin{bmatrix}1&\bm{0}\\\bm{0}&\bm{Q}'\end{bmatrix},
\quad \text{where} \; \bm{Q}'\in\mathcal{O}_{m-1,m-1}.
\]
Since $\widehat{\bm{U}}\,\widetilde{\bm{Q}}$ is a semi-orthogonal matrix and $\widehat{\bm{U}}\,\widetilde{\bm{Q}}$ has the first column $\bm{\xi}_m$,  $\widehat{\bm{U}}\,\widetilde{\bm{Q}}$ can be considered as an element in 
$\mathcal{O}_{m,m}$ whose first column equals $\bm{\xi}_{m}$.
Finally, set
\[
\bm{U} = \widehat{\bm{U}}\,\widetilde{\bm{Q}}\in\mathcal{O}_{m,m},
\qquad
\bm{V} = \widehat{\bm{V}}\in\mathcal{O}_{n,m}.
\]
By construction, the first columns of \(\bm{U}\) and \(\bm{V}\) are
\(\bm{\xi}_{m}\) and \(\bm{\xi}_{n}\), respectively, and
\(\widetilde{\bm{W}} = \bm{U}\,\bm{V}^T\).
\end{proof}
The optimal value of the objective function in~\eqref{eq:main_optimiz2} can be explicitly computed by evaluating the solution characterized in Theorem~\ref{thm:equivalence}.
Substituting the optimal solution \(  \widetilde{\bm{W}} = \bm{U} \bm{V}^T \) yields
\[
\bm{1}_m^T \bm{W} \bm{1}_n 
= \left(\bm{U}^T \bm{1}_m\right)^T \left(\bm{V}^T \bm{1}_n\right) = \sqrt{mn}.
\]
Denote the set of all optimal solutions to the maximization problem in \eqref{eq:main_optimiz2} as \( \widetilde{\mathcal{O}}_{m,n} \).
Equivalently, 
\begin{equation}
    \widetilde{\mathcal{O}}_{m,n}:= \left\{ \bm{W}\in \mathcal{O}_{m,n} \Bigl|  \frac{1}{\sqrt{mn}}\bm{1}_m^T \bm{W} \bm{1}_n=\bm{\xi}_m^T \bm{W} \bm{\xi}_n = 1 \right\}.
\end{equation}

This article proposes to initialize the weight matrix of each layer by independently sampling from $\widetilde{\mathcal{O}}_{m,n}$.

\section{Theoretical properties}

This section investigates several properties of matrices in \( \widetilde{\mathcal{O}}_{m,n} \), the set of optimal solutions to the maximization problem in \eqref{eq:main_optimiz2}.
We first analyze its theoretical characteristics, highlighting key structural and alignment features that are preserved under the orthogonality constraint.
Next, we present an efficient construction method for weight matrices in the solution set $\widetilde{\mathcal{O}}_{m,n}$ that exactly satisfy the prescribed optimization constraints.

\subsection{Matrix Structure and Characterization}
 Each matrix \( \bm{W} \in \widetilde{\mathcal{O}}_{m,n} \) admits the decomposition
\begin{equation}\label{eq:UVT}
\bm{W} = \bm{U}\bm{V}^T = 
\begin{bmatrix}
| & |
&   & | \\
\bm{u}_1 & \bm{u}_2 &  \cdots & \bm{u}_m \\
| & |
&   & | \\
\end{bmatrix} 
\begin{bmatrix} 
\text{\textemdash}\;\bm{v}_1^T\;\text{\textemdash}\\ 
\text{\textemdash}\;\bm{v}_2^T\;\text{\textemdash}\\ 
\vdots\\ 
\text{\textemdash}\;\bm{v}_m^T\;\text{\textemdash} 
\end{bmatrix},
\end{equation}
where \( \bm{U} \in \mathcal{O}_{m,m} \), \( \bm{V} \in \mathcal{O}_{n,m} \), with the leading singular vectors aligned to the normalized all-ones vectors: \( \bm{u}_1 = \xi_{m} \), \( \bm{v}_1 = \xi_{n} \).
Since every \( \bm{W} \in \widetilde{\mathcal{O}}_{m,n} \) is semi-orthogonal, it satisfies \( \bm{W} \bm{W}^{T} = \bm{I}_m \), ensuring that the transformation preserves input norms:
\[
\|
\bm{W}^T \bm{x} \| = \| \bm{x} \| \quad \text{for all }\bm{x} \in \mathbb{R}^m.
\]
On the other hand, the right-side composition \( \bm{W}^T \bm{W} = \sum_{i=1}^m \bm{v}_i \bm{v}_i^{T} \) does not yield the identity.
Instead, it defines an orthogonal projection onto the span of \( \{\bm{v}_1, \dots, \bm{v}_m\} \subset \mathbb{R}^n \).
In addition to orthogonality and projection properties, the following lemma captures a key structural characterization of the matrices in \( \widetilde{\mathcal{O}}_{m,n} \).
\begin{lemma}\label{lem:meanshift} 
Suppose that  \( \bm{W}\) is a semi-orthogonal matrix.
Then the following are equivalent:
\begin{enumerate}
    \item[$($i$)$]  \( \bm{W}\in \widetilde{\mathcal{O}}_{m,n} \),
    \item[$($ii$)$]  \( \bm{W}\bm{\xi}_n = \bm{\xi}_m \),
    \item[$($iii$)$] \( \bm{W}^T\bm{\xi}_m = \bm{\xi}_n \).
\end{enumerate}
\end{lemma}
\begin{proof}
$[(i)\Rightarrow(ii)]$
If $ \bm{W}\in \widetilde{\mathcal{O}}_{m,n} $, by Theorem~\ref{thm:equivalence}, 
$\bm{W}$ has the decomposition $\bm{W} = \bm{U}\,\bm{V}^T$ for $ \bm{U}$, $\bm{V}$ as in \eqref{eq:UVT}.
Since  $ \bm{v}_1^T\,\bm{\xi}_n = 1$ and $ \bm{v}_i^T\,\bm{\xi}_n = 0$ for all $i\ge 2$, it follows that 
$\bm{W} \bm{\xi}_n = \bm{U}\,\left(\bm{V}^T\bm{\xi}_n\right) = \bm{U} [ 1 \, 0 \, \cdots \, 0 ]^T = \,\bm{u}_1 = \bm{\xi}_m.$

\smallskip

$[(ii)\Rightarrow(iii)]$ If $\bm{W}\bm{\xi}_n = \bm{\xi}_m$, then
$1 = \bm{\xi}_m^T \bm{\xi}_m = (\bm{W}\bm{\xi}_n)^T \bm{\xi}_m = \langle\bm{\xi}_n,~ \bm{W}^T \bm{\xi}_m\rangle$, implying that $\bm{W}^T \bm{\xi}_m = \bm{\xi}_n$, since $\|\bm{W}^T \bm{\xi}_m \|
= 1$.

\smallskip

$[(iii)\Rightarrow(i)]$ $\bm{1}_n^T \bm{W}^T \bm{1}_m=\sqrt{mn} (\bm{\xi}_n^T \bm{W}^T \bm{\xi}_m)=\sqrt{mn} (\bm{\xi}_n^T \bm{\xi}_n) = \sqrt{mn}$, so $\bm{W}$ is an optimal solution of \eqref{eq:main_optimiz1}, i.e., $\bm{W}\in\widetilde{\mathcal{O}}_{m,n}$.
\end{proof}

Lemma~\ref{lem:meanshift} characterizes \( \widetilde{\mathcal{O}}_{m,n} \) as the set of all semi-orthogonal matrices that preserve the normalized all-ones vector.
Since the three conditions are equivalent, either  (ii) or (iii) in Lemma~\ref{lem:meanshift} can fully characterize \( \widetilde{\mathcal{O}}_{m,n} \).
This lemma plays a central role in the derivation of subsequent properties.
It is well known that any semi-orthonormal matrix can be extended to an orthonormal matrix when the number of rows is strictly less than the number of columns~\cite{horn2012matrix}.
Specifically, 
for a semi-orthonormal matrix \( \bm{W} \in \widetilde{\mathcal{O}}_{m,n} \), where \( m < n \), there exists a matrix \( \bm{W}^\perp \in \mathbb{R}^{(n-m) \times n} \) whose rows form an orthonormal basis for the orthogonal complement of the row space of \( \bm{W} \) such that 
$\begin{bmatrix}
    \bm{W} \\
    \bm{W}^\perp
\end{bmatrix} \in \mathcal{O}_{n,n}.$
Write
\begin{equation}\label{eq:orthogonal_1}
\bm{W} = 
\begin{bmatrix}
    \bm{w}_1^T \\
    \vdots \\
    \bm{w}_m^T
\end{bmatrix} \in \widetilde{\mathcal{O}}_{m,n}, \quad
\bm{W}^\perp = 
\begin{bmatrix}
    \bm{w}_{m+1}^T \\
    \vdots \\
    \bm{w}_n^T
\end{bmatrix} \in \mathcal{O}_{n-m,n},
\end{equation}
where \( \{ \bm{w}_i \}_{i=1}^n 
\subset \mathbb{R}^n \) forms orthonormal basis for \( \mathbb{R}^n \).

\begin{proposition}\label{prop:1}
The rows of $\bm{W}^\perp$ in \eqref{eq:orthogonal_1} are in the subspace orthogonal to $\bm{\xi}_n$, i.e.,
\[
{\bm{W}^\perp} \bm{\xi}_n \;=\;
\bm{0}.
\]
\end{proposition}
\begin{proof}
By Lemma~\ref{lem:meanshift} (ii), it follows that 
\begin{align*}
    1= \| \bm{\xi}_n\|^2
    = \left\|
\begin{bmatrix}
        \bm{W} \\ \bm{W}^\perp
    \end{bmatrix} \bm{\xi}_n \right\|^2
    = 1 + \|
{\bm{W}^\perp} \bm{\xi}_n \|^2,
\end{align*}
implying $ {\bm{W}^\perp} \bm{\xi}_n  = \bm{0}$.
\end{proof}
This proposition reveals a fundamental structural property of any full orthonormal basis extended from a matrix $\bm{W} \in \widetilde{O}_{m,n}$.
It demonstrates that the row space of the complete $n \times n$ orthogonal matrix is partitioned with respect to the all-ones vector $\bm{\xi}_n$.
The first \(m\) basis vectors (the rows of $\bm{W}$) collectively map $\bm{\xi}_n$ to $\bm{\xi}_m$, while the remaining $n-m$ basis vectors (the rows of $\bm{W}^{\perp}$) lie entirely in the subspace orthogonal to $\bm{\xi}_n$.
This complete structural understanding forms the theoretical basis for designing and verifying the construction methods that follow.

\subsection{Construction of Proposed Weight Matrices}
As stated in Theorem~\ref{thm:equivalence}, when \( m = 1 \), the set \( \widetilde{\mathcal{O}}_{1,n} \) contains only a single vector \( \bm{\xi}_n \), which leads to a fully deterministic construction.
So we focus on the more general and nontrivial cases where \( m \geq 2 \), for which the proposed algorithm is applicable.
Algorithm~\ref{algorithm} presents a construction method to generate random weight matrices from \( \widetilde{\mathcal{O}}_{m,n} \) using \eqref{eq:UVT}.
Algorithm~\ref{algorithm} first constructs two tall matrices $[\bm{\xi}_m\, \bm{A}]\in\mathbb{R}^{m\times m}$ and $[\bm{\xi}_n\, \bm{B}]\in\mathbb{R}^{n\times m}$, where $\bm{A}$ and $\bm{B}$ have i.i.d.
$\mathcal{N}(0,1)$ entries.
Then it performs thin QR decompositions on these blocks to obtain orthogonal factors $\bm{U}\in\mathbb{R}^{m\times m}$ and $\bm{V}\in\mathbb{R}^{n\times m}$, respectively.
Next, diagonal sign-correction matrices $\bm{\Lambda}$ and $\bm{\Gamma}$ are applied so that the first columns of $\bm{U}$ and $\bm{V}$ match $\bm{\xi}_m$ and $\bm{\xi}_n$, respectively.
Finally, $\bm{W} = \bm{U}\bm{V}^{T}$ yields a semi-orthogonal matrix in $\widetilde{\mathcal{O}}_{m,n}$.
In particular, $\bm{U}$ and $\bm{V}$ retain their fixed first column.
In contrast, their remaining columns are semi-orthogonal and uniformly sampled from the corresponding subspace according to the Haar distribution~\cite{mezzadri2007}.
Note that Algorithm~\ref{algorithm} requires QR decomposition for both $n\times n$ matrices and $m\times m$ matrices, respectively.
The QR decomposition of an $m\times m$ matrix has a computational complexity of $O(m^3)$.
To reduce high complexity, we propose an improved algorithm to construct $\bm{W}\in\widetilde{\mathcal{O}}_{m,n}$ with just one QR decomposition through the following process.
For \( \bm{W} \in \widetilde{\mathcal{O}}_{m,n} \), consider the set of orthonormal vectors \( \{ \bm{w}_i \}_{i=1}^m \subset \mathbb{R}^n \) such that $\bm{w}_i^T$ is the $i$-th row vector of $\bm{W}$ for all $i=1,\ldots, m$.
Define 
\begin{equation}\label{eq:u_i}
    \bm{h}_i := \bm{w}_i - \operatorname{proj}_{\bm{\xi}_n}(\bm{w}_i) \quad \text{for all }i=1,\ldots,m.
\end{equation}
Geometrically, each \(\bm{h}_i\) is the orthogonal projection
of \(\bm w_i\) onto the \((m-1)\)-dimensional subspace
$\{\bm x\in\mathbb R^n : \bm{\xi}_n^T\bm{x} = 0\}$.

\begin{proposition}\label{prop:2}
Let $n\geq m\geq 2$.
Consider $\bm{H} = [\bm{h}_1 \, \bm{h}_2 \, \cdots \, \bm{h}_m]$ where $\bm{h}_i$ is defined in \eqref{eq:u_i}, then Gram matrix $\bm{H}^T \bm{H}$ has following property:
\begin{align*}
    (\bm{H}^T\bm{H})_{ij}= \bm{h}^T_i\bm{h}_j = \begin{cases} \frac{m-1}{m} & \text{if } i = j, \\[2mm]
 - \frac{1}{m} & \text{if } i \neq j
\end{cases}
\end{align*}
for all $i,j=1,\ldots, m$.
\end{proposition}
\begin{proof}
By Lemma~\ref{lem:meanshift} it holds that for $i = 1, \ldots,m$,
\begin{equation}\label{eq:h_w}
    \bm{h}_i
    = \bm{w}_i - (\bm{w}_i^T\bm{\xi}_n)\bm{\xi}_n
    = \bm{w}_i - \sqrt{\frac{1}{m}}\bm{\xi}_n.
\end{equation}
Then it follows that
    \begin{align*}
        \bm{h}_i^T\bm{h}_j 
        &= \left(\bm{w}_i - \sqrt{\frac{1}{m}}\bm{\xi}_n \right)^T \left(\bm{w}_j - \sqrt{\frac{1}{m}}\bm{\xi}_n \right)\\
        &= \bm{w}_i^T\bm{w}_j - {\frac{1}{m}}\bm{\xi}_n^T\bm{\xi}_n\\\
        &= \begin{cases}
    1 - \frac{1}{m} 
      = \frac{m-1}{m}, & \text{if } i = j, \\[2mm]
     - \frac{1}{m}, & \text{if } i \neq j
    \end{cases}
    \end{align*}
    for all $i,j=1,\ldots, 
m$.
\end{proof}

For $m\geq 2$, let
\begin{equation}
    \bm{P}_{m}:= \bm{H}^T\bm{H} =  \bm{I}_{m}-\frac{1}{m}\bm{J}_{m}.
\end{equation}
Consider multiplying $\bm{P}_{m}$ by an \(m\)-dimensional vector $\bm{x}\in\mathbb{R}^{m}$,
\[
\bm{P}_{m}\bm{x} = (\bm{I}_{m}- \bm{\xi}_{m}\bm{\xi}_{m}^{T})\bm{x} = \bm{x} - \bm{\xi}_{m}(\bm{\xi}_{m}^{T} \bm{x}).
\]
This subtracts the component of the original vector $\bm{x}$ in the direction of $\bm{\xi}_{m}$, which geometrically corresponds exactly to projecting $\bm{x}$ onto $\bm{\xi}_{m}^{\perp}$.
Moreover, one can see that $\bm{P}_m$ is an $m\times m$ symmetric positive semidefinite matrix.
So, $\bm{P}_{m}$ admits a Cholesky decomposition $\bm{P}_{m} = \bm{L}\bm{L}^{T}$ where $\bm{L}$ is a lower triangular matrix with nonnegative diagonal entries.

\begin{algorithm}[t]
\caption{Generate $\bm{W} \in \widetilde{\mathcal{O}}_{m,n}$ via QR factorization}\label{algorithm}
\begin{algorithmic}[1]
  \STATE \textbf{Input:} Positive integers \( n \geq m \geq 2\)
  \STATE \textbf{Output:} \(\bm{W}\)
  \STATE Draw \(\bm{A}\in\mathbb{R}^{m\times(m-1)}\) and \(\bm{B}\in\mathbb{R}^{n\times(m-1)}\) with i.i.d.\ \(\mathcal{N}(0,1)\) entries
  \STATE \((\bm{U},\bm{R})\gets \mathrm{qr}\bigl([\bm{\xi}_m\;\;\bm{A}]\bigr)\) 
    \hfill\(\,\bm{U}\in\mathbb{R}^{m\times m},\;\bm{R}\in\mathbb{R}^{m\times m}\)
  \STATE \((\bm{V},\bm{S})\gets \mathrm{qr}\bigl([\bm{\xi}_n\;\;\bm{B}]\bigr)\) 
    \hfill\(\,\bm{V}\in\mathbb{R}^{n\times m},\;\bm{S}\in\mathbb{R}^{m\times m}\)
  \STATE \(\bm{\Lambda}\gets \bm{0}_{m\times m}\)
  \STATE \(\bm{\Gamma}\gets \bm{0}_{m\times m}\)
  \FOR{\(i = 1,\dots,m\)}
    \STATE \(\bm{\Lambda}_{i,i} \gets \bm{R}_{i,i} / \lvert \bm{R}_{i,i}\rvert\)
    \STATE \(\bm{\Gamma}_{i,i} \gets \bm{S}_{i,i} / \lvert \bm{S}_{i,i}\rvert\)
  \ENDFOR
  \STATE \(\bm{U} \gets \bm{U}\,\bm{\Lambda}\)
  \STATE \(\bm{V} \gets \bm{V}\,\bm{\Gamma}\) 
 
  \STATE \(\bm{W} \gets \bm{U}\,\bm{V}^T\)  
\end{algorithmic}
\end{algorithm}

\begin{lemma}\label{lemma:cholBm}
Let \( m \geq 2 \). Then \( \bm{P}_{m} = \bm{I}_{m} - \frac{1}{m} \bm{J}_{m} \in \mathbb{R}^{m \times m} \) admits a unique Cholesky factorization 
\[
\bm{P}_{m} = \bm{L} \bm{L}^T,
\]
where \( \bm{L} \in \mathbb{R}^{m \times m} \) is a lower triangular matrix with nonnegative diagonal entries, explicitly given by
\begin{equation}\label{eq:L_struct}
    L_{ij} =
    \begin{cases}
    \sqrt{\dfrac{m - i}{m - i + 1}}, 
    & \text{if } i = j < m,\\[1ex]
    -\dfrac{1}{\sqrt{(m - j + 1)(m - j)}}, 
    & \text{if } 1 \leq j < i 
\leq m,\\[1ex]
    0, 
    & \text{otherwise}.
\end{cases}
\end{equation}
\end{lemma}
\begin{proof}
    The uniqueness follows from the fact that all leading principal minors of \( \bm{P}_m \) are positive, ensuring the existence of a unique Cholesky factor with nonnegative diagonal entries~\cite{higham2002accuracy}.
\end{proof}

Since the Cholesky factor \( \bm{L} \) of \( \bm{P}_m \) is explicitly available, we can construct the matrix \( \bm{H}^T = \bm{L}\bm{Q} \) for a certain semi-orthogonal matrix \( \bm{Q} \in \mathcal{O}_{m,n} \).
Using the relation in~\eqref{eq:h_w},
\[
\bm{w}_i = \bm{h}_i + \sqrt{\frac{1}{m}}\bm{\xi}_n
\qquad \text{for} \;i=1,\dots,m.
\]
In matrix form, it can be expressed as 
\begin{align*}
\bm{W}
= \bm{H}^T + \frac{1}{\sqrt{mn}}\bm{J}_{m\times n}.
\end{align*}
The following theorem provides the mathematical foundation for this decomposition and guarantees that all such matrices satisfy the defining conditions of \( \widetilde{\mathcal{O}}_{m,n} \).

\begin{theorem}\label{thm:generate W2}
Let \( m \geq 2 \). 
\begin{equation}\label{eq:LQ}
   \widetilde{\mathcal O}_{m,n}=\left\{\! \bm{L}\,\bm{Q}+\frac{1}{\sqrt{mn}}\bm{J}_{m\times n} \bigl| \bm{Q}\in \mathcal{O}_{m,n}, \bm{Q}^T\bm{e}_1=\bm{\xi}_n \!
\right\},
\end{equation}
where \( \bm{L} \) is the $m\times m$ lower triangular matrix defined in Lemma~\ref{lemma:cholBm}.
\end{theorem}

\begin{algorithm}[!ht]
\caption{Generate $\bm{W} \in \widetilde{\mathcal{O}}_{m,n}$ via Theorem~\ref{thm:generate W2}}\label{alg:algorithm2}
\begin{algorithmic}[1]
  \STATE \textbf{Input:} Positive integers \( n \geq m \geq 2\)
  \STATE \textbf{Output:} \(\bm{W}\)
  \STATE \(\bm{L}_{ij} = \begin{cases} \sqrt{(m-i)/(m-i+1)} & \text{if } i=j<m \\ -1/\sqrt{(m-j+1)(m-j)} & \text{if } j<i \\ 0 & \text{otherwise for } j \le i \end{cases}\)
  \STATE Draw \(\bm{A} \in \mathbb{R}^{n\times(m-1)}\) with i.i.d.\ \(\mathcal{N}(0,1)\) entries
  \STATE \(({\bm Q} , \bm{R}) \gets \operatorname{qr}([\bm{\xi}_n \;\; \bm{A}])\) 
  \STATE \(\bm{\Lambda}\gets \bm{0}_{m\times m}\)
  \FOR{\(i = 1,\dots,m\)}
    \STATE \(\bm{\Lambda}_{i,i} \gets \bm{R}_{i,i} / \lvert \bm{R}_{i,i}\rvert\)
  \ENDFOR
  \STATE \({\bm Q}  \gets {\bm Q} \,\bm{\Lambda}\)
  
\STATE \(\bm{Q} \gets  [\bm{q}_2 ~ \cdots ~ \bm{q}_m ~ \bm{q}_1]^T\), where \(\bm{q}_i\) : $i$-th columns of \({\bm Q} \).
\STATE \(\bm{W} \gets \bm{L}\bm{Q} + \bm{\xi}_m\bm{\xi}_n^T\) 
\end{algorithmic}
\end{algorithm}
\begin{proof}
    First, we prove the reverse inclusion ($\supseteq$).
Let $\bm{W}$ be an arbitrary matrix from the set on the right-hand side, such that $\bm{W} = \bm{L}\bm{Q} + \frac{1}{\sqrt{mn}}\bm{J}_{m\times n}$, where $\bm{Q} \in \mathcal{O}_{m,n}$ is a semi-orthogonal matrix whose \(m\)-th row is $\bm{\xi}_n^T$.
By Lemma~\ref{lem:meanshift} it is enough to show that $\bm{W}\bm{\xi}_n = \bm{\xi}_m$ and $\bm{W} \bm{W}^T = \bm{I}_m$.
Since the rows of $\bm{Q}$ are orthonormal and its \(m\)-th row is $\bm{\xi}_n^T$, the vector $\bm{Q}\bm{\xi}_n$ is equal to $\bm{e}_m = [0, \dots, 0, 1]^T$.
Then it follows that 
    \begin{align*}
        \bm{W}\bm{\xi}_n 
        &=\bm{L}(\bm{Q}\bm{\xi}_n) + \frac{1}{\sqrt{mn}}\bm{J}_{m\times n}\, \bm{\xi}_n\\
        &= \bm{L}\bm{e}_m + \bm{\xi}_m = \bm{0} + \bm{\xi}_m = \bm{\xi}_m.
\end{align*}
    Next, to show that $\bm{W}$ is semi-orthogonal, we compute the product $\bm{W}\bm{W}^T$.
Since $\bm{L}\bm{Q}\bm{J}^T = \sqrt{n}\bm{L}(\bm{Q}\bm{\xi}_n)\bm{1}_m^T = \sqrt{n}\bm{L}\bm{e}_m\bm{1}_m^T = \bm{0}$,
    \begin{align*}
        \bm{W}\bm{W}^T
        & = \left(\bm{L}\bm{Q} + \frac{1}{\sqrt{mn}}\bm{J}_{m\times n}\right) \left(\bm{Q}^T\bm{L}^T + \frac{1}{\sqrt{mn}}\bm{J}_{m\times n}^T\right)\\
        & = \bm{L}\bm{Q}\bm{Q}^T\bm{L}^T + \frac{1}{mn}\bm{J}_{m\times n}\bm{J}_{m\times n}^T\\
        & = \left(\bm{I}_m - \frac{1}{m}\bm{J}_m \right) + \frac{1}{mn}(n\bm{J}_m) = \bm{I}_m.
\end{align*}
    Having shown that $\bm{W}$ is a semi-orthogonal matrix satisfying the conditions of Lemma~\ref{lem:meanshift}, it implies that $\bm{W}\in \widetilde{\mathcal{O}}_{m,n}$.
Next, we prove the forward inclusion ($\subseteq$).
    Let $\bm{W}$ be an arbitrary matrix in $\widetilde{\mathcal{O}}_{m,n}$.
Define the matrix $\bm{M} = \bm{W} - \frac{1}{\sqrt{mn}}\bm{J}_{m\times n}$.
    By using the properties of $\bm{W}$ from Lemma~\ref{lem:meanshift}, we compute the product $\bm{M}\bm{M}^T$ as follows:
    \begin{align*}
        \bm{M}\bm{M}^T 
        & = \left(\bm{W} - \frac{1}{\sqrt{mn}}\bm{J}_{m\times n}\right)\left(\bm{W}^T - \frac{1}{\sqrt{mn}}\bm{J}_{m\times n}^T\right)\\
        & = \bm{I}_m - \frac{1}{m}\bm{J}_m = \bm{P}_m.
\end{align*}
    From Lemma~\ref{lemma:cholBm}, we have $\bm{M}\bm{M}^T = \bm{P}_m = \bm{L}\bm{L}^T$.
It follows from \cite[Theorem~7.3.11]{horn2012matrix} that there exists a semi-orthogonal matrix $\bm{Q} \in \mathcal{O}_{m,n}$ such that $\bm{M} = \bm{L}\bm{Q}$.
To determine the properties of $\bm{Q}$, we use the property from Lemma~\ref{lem:meanshift} that $\bm{W}\bm{\xi}_n = \bm{\xi}_m$, which implies $\bm{M}\bm{\xi}_n = \bm{W}\bm{\xi}_n - \frac{1}{\sqrt{mn}}\bm{J}\bm{\xi}_n = \bm{\xi}_m - \bm{\xi}_m = \bm{0}$.
Substituting the factorization gives $\bm{L}(\bm{Q}\bm{\xi}_n) = \bm{0}$.
    From Lemma~\ref{lemma:cholBm}, the last column of $\bm{L}$ is zero.
Thus, the product $\bm{M} = \bm{L}\bm{Q}$ implies that $\bm{M}$ is determined solely by the first $m-1$ rows of $\bm{Q}$.
The property $\bm{W} \in \widetilde{\mathcal{O}}_{m,n}$ requires $\bm{M}\bm{\xi}_{n}=\bm{0}$.
    This yields a system of linear equations for the terms $x_j = \bm{q}_j^{T}\bm{\xi}_n$ for $j=1,2,\ldots,m-1$.
Since the first $m-1$ columns of $\bm{L}$ are linearly independent, $\bm{q}_j^{T}\bm{\xi}_n = 0$ for all $j=1,2,\ldots,m-1$.
For $\bm{q}_{m}^{T}$, the last row of $\bm{Q}$, it must be a unit vector orthogonal to $\{\bm{q}_{1}^{T} ,\ldots,\bm{q}_{m-1}^{T} \}$.
For the decomposition in \eqref{eq:LQ} to be a unique characterization, we constrain the choice of the factorization such that $\bm{q}_{m}^{T}$ is selected from the remaining 1-dimensional space spanned by $\bm{\xi}_n$.
This leads to $\bm{q}_{m}^{T} = \pm \bm{\xi}_n$.
    By convention, we choose the positive sign for the set.
Thus, any $\bm{W} \in \widetilde{\mathcal{O}}_{m,n}$ can be decomposed into the desired form.
\end{proof}

\subsubsection*{Computational Complexity}
While both Algorithm~\ref{algorithm} and Algorithm~\ref{alg:algorithm2} provide methods for generating a matrix $\bm{W} \in \widetilde{\mathcal{O}}_{m,n}$, the approach in Algorithm~\ref{alg:algorithm2}, derived from Theorem~\ref{thm:generate W2}, is significantly more efficient.
Algorithm~\ref{algorithm} requires two separate QR decompositions: one for an $m \times m$ matrix and another for an $n \times m$ matrix, resulting in a computational complexity of about $O(m^3 + nm^2)$.
In contrast, Algorithm~\ref{alg:algorithm2} first computes the fixed lower-triangular matrix $\bm{L}$ using a direct closed-form expression from Lemma~\ref{lemma:cholBm}, which costs only $O(m^2)$.
It then performs a single QR decomposition on an $n \times m$ matrix to build the semi-orthogonal matrix $\bm{Q}$.
This approach reduces the total complexity to $O(m^2 + nm^2)$, removing the $O(m^3)$ term entirely.
When the ratio $m/n$ is close to $1$, this eliminated cost becomes a substantial portion of the total computation, potentially cutting the initialization time in half.

\section{Statistical properties}

In this section, the statistical characteristics of matrices in \( \widetilde{\mathcal{O}}_{m,n} \) are examined, together with their implications for signal propagation in ReLU‑based neural networks.
The discussion begins with the linear transformation \( f(\bm{x}) = \bm{W}\bm{x} \) for a weight matrix \( \bm{W} \in \widetilde{\mathcal{O}}_{m,n} \), with emphasis on its influence over key input-output statistics-most notably covariance structure and distributional behavior in high‑dimensional regimes.
The scope is then extended to multi‑layer ReLU networks. A heuristic mean-field framework~\cite{pennington2017resurrecting,yang2017mean} is adopted to trace the evolution of activation statistics across layers.

\subsection{Linear Transform Behavior}
Before delving into the implications of the proposed initialization scheme for ReLU-based neural networks, we first examine a fundamental aspect of forward propagation: the statistical behavior of the linear transformation \( f(\bm{x})= \bm{W} \bm{x} \) for the proposed weight matrix $\bm{W} \in \widetilde{\mathcal{O}}_{m,n}$.
We aim to characterize the statistical properties of \( \bm{W} \) and its effect on the image of $\bm{W}$, as it plays a central role in shaping the dynamics of signal propagation through fully connected layers.
We denote by \( \mathbb{E}[\bm{x}] \) the expectation of the random vector \( \bm{x} \), and by \( \mathrm{Cov}[\bm{x}] \) its covariance matrix (for a scalar random variable \(X\), \(\mathrm{Cov}[X]\) simply becomes \(\mathrm{Var}[X]\)).
We write $X \stackrel{d}{=} Y$ to denote that $X$ and $Y$ are equal in distribution.
\begin{proposition}\label{pro:clt}
Let $\bm{W} \in \widetilde{\mathcal{O}}_{m,n}$ be given.
For a random vector $\bm{x} \in \mathbb{R}^n$ with $\mathbb{E}[\bm{x}]=\mu \bm{1}_n$, $\operatorname{Cov}[\bm{x}]=\sigma^2\bm{I}_n,$ ($\mu\in \mathbb{R}$, $\sigma>0$),
it holds that
\[
\mathbb{E}[\bm{W}\bm{x}] = \mu \sqrt{\frac{n}{m}}\bm{1}_m ,\quad \operatorname{Cov}[\bm{W}\bm{x}] = \sigma^2\bm{I}_m.
\]
\end{proposition}
\begin{proof}
By the linearity of expectation and the decomposition \( \bm{x} = (\bm{x} - \mu \bm{1}_n) + \mu \bm{1}_n \), one can obtain that
\begin{align*}
  \mathbb{E}\left[\bm{W} \bm{x}\right]
&=\mathbb{E}\left[\bm{W} \left(\bm{x}-\mu\bm{1}_n+\mu\bm{1}_n \right) \right]\\
&=\mathbb{E}\left[\bm{W} \left(\bm{x}-\mu\bm{1}_n\right)+ \mu\bm{W}\bm{1}_n  \right]\\
&=\bm{W}\mathbb{E}\left[ \bm{x}-\mu\bm{1}_n \right] +\mu \mathbb{E}\left[\bm{W}\bm{1}_n  \right]\\
&=\mu \bm{W}\bm{1}_n  \\
&=\mu \sqrt{\frac{n}{m}}\bm{1}_m.
\end{align*}
The last equality holds from Lemma~\ref{lem:meanshift}. 
And by the linearity of expectation, it holds that 
\begin{align*}
\mathrm{Cov}[\bm{W}\bm{x}] 
&= \mathbb{E}\left[ (\bm{W} \bm{x} - \bm{W} \mathbb{E}[\bm{x}]) (\bm{W} \bm{x} - \bm{W} \mathbb{E}[\bm{x}])^T \right] \\
&= \mathbb{E}\left[ \bm{W} (\bm{x} - \mathbb{E}[\bm{x}]) (\bm{x} - \mathbb{E}[\bm{x}])^T \bm{W}^T \right] \\
&= \bm{W} \, \mathbb{E} \left[ (\bm{x} - \mathbb{E}[\bm{x}]) (\bm{x} - \mathbb{E}[\bm{x}])^T \right] \bm{W}^T \\
&= \bm{W} \, \mathrm{Cov}[\bm{x}] \, \bm{W}^T\\
&= \sigma^2\bm{I}_m.
\end{align*}
Since $\bm{W}$ is a semi-orthogonal matrix, the last equality holds.
\end{proof}

Proposition~\ref{pro:clt} shows that when the input vector \( \bm{x} \in \mathbb{R}^n \) has a constant mean \( \mathbb{E}[\bm{x}] = \mu \bm{1}_n \) and isotropic covariance \( \mathrm{Cov}[\bm{x}] = \sigma^2 \bm{I}_n \), the output \(\bm{W} \bm{x} \) under a semi-orthogonal transformation \( \bm{W} \in \widetilde{\mathcal{O}}_{m,n} \) also has an explicit mean and covariance:
\[
\mathbb{E}[\bm{W} \bm{x}] = \mu \sqrt{\frac{n}{m}} \bm{1}_m, \quad \mathrm{Cov}[\bm{W} \bm{x}] = \sigma^2 \bm{I}_m.
\]
This result indicates that the transformation preserves isotropy and uniformly rescales the mean, which is beneficial for stabilizing signal propagation in deep networks.
If we further assume that \( \bm{x} \) follows a multivariate normal distribution \( \mathcal{N}(\mu \bm{1}_n, \sigma^2 \bm{I}_n) \), then the transformed vector \(\bm{W} \bm{x} \) also follows a multivariate normal distribution with the same mean and covariance as derived above, due to the affine invariance of Gaussian distributions.
However, this Gaussian assumption may not hold in real-world data.
Empirical evidence shows that many practical datasets exhibit significant departures from normality, such as skewness, heavy tails, or outliers~\cite{blanca2013skewness, raymaekers2024transforming}.
Consequently, the exact Gaussian form of the output \( \bm{y} \) is not guaranteed in general.
Let \( \mathbb{P}(\mathcal{E}) \) denote the probability of an event \( \mathcal{E} \) under a given probability space.
Recall that the cumulative distribution function (CDF) of the standard normal distribution is defined by
\[
\Phi(z) := \int_{-\infty}^{z} \frac{1}{\sqrt{2\pi}} \exp\left(-\frac{t^2}{2}\right) \, dt, \quad z \in \mathbb{R}.
\]
We write \( X_n \overset{d}{\longrightarrow} X \) to denote convergence in distribution of a sequence of random variables \( \{X_n\} \) to a random variable \( X \), that is,
\[
X_n \overset{d}{\longrightarrow} X \quad \text{if and only if} \quad \lim_{n \to \infty} \mathbb{P}(X_n \le z) = \mathbb{P}(X \le z)
\]
for all \( z \in \mathbb{R} \) at which the cumulative distribution function of \( X \) is continuous.
And we write \( X_n \overset{p}{\longrightarrow} X \) to denote convergence in probability, for $\epsilon > 0$,
\[
X_n \overset{p}{\longrightarrow} X \quad \text{if and only if} \quad \lim_{n \to \infty} \mathbb{P}(|X_n - X| > \epsilon) = 0.
\]

Recall that the Stiefel manifold is defined as $\mathcal{O}_{n,m}=\{ \bm{Q} \in \mathbb{R}^{n\times m} |
\; \bm{Q}^{T}\bm{Q} =\bm{I}_m\}$.
The elements of the Stiefel manifold are sometimes called \(m\)--frames in $\mathbb{R}^{n}$.
Let $\bm{A} \in \mathbb{R}^{n \times m}$ be a standard Gaussian random matrix.
We perform the \emph{QR decomposition} of $\bm{A}$, that is, $\bm{A} = \bm{Q}\bm{R}$, where 
$\bm{Q} \in O_{n,m}$ lies on the Stiefel manifold and 
$\bm{R} \in \mathbb{R}^{m \times m}$ is an upper triangular matrix with nonnegative diagonal entries.
It is well known that, under this construction, $\bm{Q}$ is uniformly distributed on the Stiefel manifold $\mathcal{O}_{n,m}$~\cite{tropp2012comparison}.
Using the result of~\cite{jiang2005maxima}, the following lemma can be obtained.
\begin{lemma}\label{lem:constant}
    Let $\widehat {\bm Q}=[\widehat{Q}_{ij}] \in \mathcal{O}_{n,m-1}$ be a uniformly distributed random $(m-1)$--frame in the subspace $\bm{\xi}_n^{\perp}$.
Then for any \( C > 2 \),
    \begin{equation}\label{eq:Q_prob}
        \mathbb{P} \left( \max_{\substack{1\le i \le n \\ 1 \le j \le m-1}} \left|\widehat Q_{ij}\right|\geq C\sqrt{\frac{\log n}{n}} \right)\rightarrow 0
    \end{equation}
    as \(n\rightarrow \infty\).
\end{lemma}

Note that $\bm Q= [\widehat{\bm Q}\; \xi_n]^{T} \in \mathcal{O}_{m,n}$, where $\widehat{\bm Q}$ in Lemma~\ref{lem:constant} coincides with the matrix $\bm Q$ appearing in line~11 of Algorithm~\ref{alg:algorithm2}.
Using Lemma~\ref{lem:constant}, the following theorem holds.

\begin{theorem}\label{lem: big_O}
    Let \( \bm{W} = [W_{ij}]\) be a random matrix as \( \bm{W} = \bm{L}\bm{Q} +  \frac{1}{\sqrt{mn}}\bm{J}_{m\times n} \), where \( \bm{L} \) is the fixed matrix from Lemma~\ref{lemma:cholBm} and $\bm Q=[\widehat{\bm Q}\;
\xi_n]^{T} \in \mathcal{O}_{m,n}$ where $\widehat{\bm Q} $ is a uniformly distributed random $(m-1)$--frame in the subspace $\bm{\xi}_n^{\perp}$.
Then for any \( C  > 2 \) and fixed \(m\),
    \begin{equation}\label{eq:W_prob}
        \mathbb{P} \left( \max_{\substack{1\le i \le m \\ 1 \le j \le n}} \left|W_{ij}\right|\geq C\sqrt{\frac{\log n}{n}} + \frac{1}{\sqrt{mn}} \right)\rightarrow 0
    \end{equation}
    as \(n\rightarrow \infty\).
\end{theorem}
\begin{proof}
Since $\widehat{\bm Q}$ is uniformly distributed as an $(m-1)$--frame in 
the subspace $\bm{\xi}_n^{\perp}$, it is invariant under left multiplication by any $\bm R \in \mathcal O_{m-1}$, i.e., $ \bm R \widehat{\bm Q}^T\stackrel d=  \widehat{\bm Q}^T$.
For any nonzero vector $\bm v\in \mathbb R^{m-1}$, let $\bm R \in \mathcal{O}_{m-1}$ be such that $\frac{\bm v}{\|\bm v\|}= \bm R^T \bm e_1 $ , where $\bm e_1=[1~0~\cdots~0]^T$ is the first standard basis vector.
Then
\begin{equation}\label{eq:vq}
 \bm v^T\widehat{ \bm Q}^T
=  \|\bm v\|\bm e_1^T \bm R\widehat{ \bm Q}^T
\stackrel{d}{=} \|\bm v\|\bm e_1^T \widehat{ \bm Q}^T
= \|\bm v\|\bm q_1^T,
\end{equation}
where $\bm q_1$ is the first column of $\widehat{\bm Q}$.
Now denote $\bm L = [\,\widehat{\bm L}\;\;\mathbf 0\,]$, and denote by $\bm l_i^{T}$ the $i$-th row of $\widehat{\bm L}$.
Using \eqref{eq:vq}, $(\bm L\bm Q)_{i} = \bm l_i^{T}\widehat{\bm Q}^T \stackrel{d}{=}\ \|\bm l_i \|
\bm q_1^T,$ where $(\bm L\bm Q)_{i}$ is the $i$-th row of $\bm L \bm Q$.
Since $\|\bm l_i\|_2^2 = (\bm L\bm L^{T})_{ii} = 1-\tfrac{1}{m}$, it follows that
\[
(\bm L\bm Q)_{i}
\stackrel{d}{=}  \sqrt{1-\frac1m}  \bm q_1^T.
\]
For $i = 1, \dots, m$,
\begin{align*}
    & \;\mathbb{P} \left(  \max_{\substack{ 1 \le j \le n}} \left|(\bm L \bm Q)_{ij}\right| \geq C \sqrt{\frac{\log n}{n}} \right)\\
     \leq& \;
\mathbb{P} \left(   \left\|\bm q_1 \right\|_{\infty} \geq C \sqrt{\frac{\log n}{n}} \right) \rightarrow 0 \quad \text{as } \;
n\rightarrow \infty
\end{align*}
using Lemma~\ref{lem:constant}.
Then
\begin{align*}
& \;\mathbb{P} \left( \max_{\substack{1\le i \le m \\ 1 \le j \le n}} \left|W_{ij}\right|\geq C\sqrt{\frac{\log n}{n}} + \frac{1}{\sqrt{mn}} \right)\\
= & \;
\mathbb{P} \left( \max_{\substack{1\le i \le m \\ 1 \le j \le n}} \left|(\bm L \bm Q)_{ij}\right|\geq C\sqrt{\frac{\log n}{n}}  \right)\\
\leq & \;
\sum_{i=1}^m \mathbb{P} \left( \max_{\substack{ 1 \le j \le n}} \left|(\bm L \bm Q)_{ij}\right|\geq C\sqrt{\frac{\log n}{n}}  \right)\\
=& \;
m \mathbb{P} \left( \max_{\substack{ 1 \le j \le n}} \left|(\bm L \bm Q)_{ij}\right|\geq C\sqrt{\frac{\log n}{n}}  \right)\rightarrow 0
\end{align*}
as \(n\rightarrow \infty\).
\end{proof}

Note that $\bm{W}$ in Theorem~\ref{lem: big_O} exactly matches $\bm{W}$ generated by Algorithm~\ref{alg:algorithm2}.
\begin{lemma}{\cite[Theorem 1.1]{raivc2019multivariate}}\label{lem:berry}
Let $\bm{y}_{1},\ldots,\bm{y}_{n}$ be independent mean $\bm{0}$ random vectors in $\mathbb{R}^{m}$ such that $\sum_{i=1}^{n}\bm{y}_{i}$ has the identity covariance matrix.
Let $\bm{g}$ be the standard Gaussian random vector in $\mathbb{R}^{m}$.
Then for all measurable convex sets $\mathcal{A} \in \mathbb{R}^{m}$,
\begin{equation}
    \left|\mathbb{P}\left( \sum_{i=1}^{n}\bm{y}_{i} \in \mathcal{A} \right) - \mathbb{P}(\bm{g} \in \mathcal{A}) \right|
\leq (42m^{1/4} + 16) \sum_{i=1}^{n}\mathbb{E}\| \bm{y}_{i} \|_{2}^{3}.
\end{equation}
\end{lemma}

\begin{theorem}\label{thm:clt}
Fix $m\in \mathbb{N}$.
Let $\bm{x} = (x_1, \dots, x_n) \in \mathbb{R}^n$ be a random vector with independent components satisfying
\[
\mathbb{E}[x_j] = \mu,\;
\operatorname{Var}(x_j) = \sigma^2,\; \mathbb{E}[|x_j|^3]<\infty\; \text{ for all } j = 1, \dots, n.
\]
Let \( \bm{W}=[ W_{ij} ] \in \widetilde{\mathcal{O}}_{m,n}\) be a random matrix as defined in Theorem~\ref{lem: big_O}, independent of $\bm{x}$.
Then
\[
\frac{\bm W\,\bm x  - \mu\sqrt{\tfrac{n}{m}}\mathbf1_m}{\sigma}\;\xrightarrow{d}\;\mathcal N\left(\bm 0 ,\;\bm{I}_m\right)
\]
as \(n\to\infty\).
\end{theorem}
\begin{proof}
Let $\bm{z}=(z_1,\ldots,z_n)$ such that $z_j:=(x_j-\mu)/\sigma$ for all $j$.
Then $\mathbb{E}[z_{j}]=0$, $\operatorname{Var}(z_{j})=1$, and $\beta_3:=\mathbb{E}[|z_j|^3]<\infty$.
For each $j=1,\ldots,n$, let $\bm{w}_{j} \in \mathbb{R}^{m}$ be the $j$-th column vector of $\bm{W}$.
For $j=1,\ldots,n$, define $\bm{y}_{j}:=\bm{w}_{j}z_{j}$.
Since the $z_j$ are independent scalars, the vectors $\bm{y}_{1},\ldots,\bm{y}_{n}$ are independent and $\mathbb{E}[\bm{y}_{j}] = \bm{0}$ and
\[
\sum_{j=1}^{n} \operatorname{Var}(\bm{y}_j) 
= \sum_{j=1}^{n}\mathbb{E}[z_{j}^{2}]\bm{w}_{j}\bm{w}_{j}^{T} 
= \sum_{j=1}^{n}\bm{w}_{j}\bm{w}_{j}^{T} = \bm{W}\bm{W}^T =\bm{I}_{m}.
\]
Therefore, Lemma~\ref{lem:berry} applies to $\bm{y}_{j}$.
For all measurable convex sets $\mathcal{A} \subset \mathbb{R}^{m}$,
\[
\sup_{\mathcal{A}}\left|\mathbb{P}\left( \sum_{i=1}^{n}\bm{y}_{i} \in \mathcal{A} \right) - \mathbb{P}(\bm{g} \in \mathcal{A}) \right|
\leq (42m^{1/4} + 16) \sum_{i=1}^{n} \mathbb{E}\| \bm{y}_{i} \|_{2}^{3},
\]
where $\bm{g}\sim\mathcal{N}(\bm{0},\bm{I}_m)$.

Since $\| \bm{y}_{j} \|_{2} = |z_{j}| \| \bm{w}_{j} \|_{2}$, $ \sum_{i=1}^{n}\mathbb{E}[\|
\bm{y}_{i} \|_{2}^{3}] = \beta_{3} \sum_{i=1}^{n}  \mathbb{E}[\| \bm{w}_{j} \|_{2}^{3}]$.
And $\sum_{i=1}^{n}  \|
\bm{w}_{j} \|_{2}^{3}$ can be bounded as follows:
\begin{align*}
    \sum_{i=1}^{n}  \|
\bm{w}_{j} \|_{2}^{3}
    &\leq \left( \max_{1\leq j\leq n} \| \bm{w}_{j} \|_2 \right) \sum_{i=1}^{n}  \|
\bm{w}_{j} \|_{2}^{2}\\
    &= \left( \max_{1\leq j\leq n} \| \bm{w}_{j} \|_2 \right)\operatorname{tr} \left( \sum_{i=1}^{n}  \bm{w}_{j} \bm{w}_{j}^{T} \right)\\
    &= m \cdot \max_{1\leq j\leq n} \|
\bm{w}_{j} \|_2\\
    &\leq m^{3/2} \max_{i,j} |W_{ij}|.
\end{align*}
It implies that $\sum_{i=1}^{n}  \mathbb{E}[\|
\bm{w}_{j} \|_{2}^{3}] \leq m^{3/2} \mathbb{E}[\max_{i,j} |W_{ij}|]$.
Since $\max_{i,j} |W_{ij}| < 1$ and $\max_{i,j} |W_{ij}|\overset{p}{\longrightarrow}0$, by uniform integrability and Theorem 2.20 in \cite{van2000asymptotic}, we have $\mathbb{E}[\max_{i,j} |W_{ij}|]\rightarrow 0$.
Let $K = (42m^{1/4} + 16) \beta_{3} m^{3/2}$, then
\[
\sup_{\mathcal{A}}\left|\mathbb{P}\left( \bm{W}\bm{z} \in \mathcal{A} \right) - \mathbb{P}(\bm{g} \in \mathcal{A}) \right| 
\leq K\;
\mathbb{E}[\max_{i,j} |W_{ij}|]\rightarrow 0
\]
as $n\rightarrow \infty$.
For any $\bm{a}\in\mathbb{R}^m$ and $t\in\mathbb{R}$, consider the closed half-space $\mathcal{H}_{\bm{a},t}:=\{\bm{x}\in\mathbb{R}^m | \ \bm{a}^{T}\bm{x}\le t\}$.
Since $\mathcal{H}_{\bm{a},t}$ is convex,
\begin{align*}
    &\;\left|\mathbb{P}(\bm{a}^{T}\bm{W}\bm{z} \leq t)-\mathbb{P}(\bm{a}^{T}\bm{g}\leq t)\right| \\
    =& \;
\left|\mathbb{P}(\bm{W}\bm{z}\in \mathcal{H}_{\bm{a},t})-\mathbb{P}(\bm{g}\in \mathcal{H}_{\bm{a},t})\right|\\
    \leq&\;\sup_{\mathcal{A}}\left|\mathbb{P}\left( \bm{W}\bm{z} \in \mathcal{A} \right) - \mathbb{P}(\bm{g} \in \mathcal{A}) \right|
\rightarrow 0.
\end{align*}
By the Cram{\'e}r-Wold theorem~\cite[Proposition 2.17]{van2000asymptotic}, as $n\rightarrow \infty$,
\[
\frac{\bm W\,\bm x  - \mu\sqrt{\tfrac{n}{m}}\mathbf1_m}{\sigma}\;\xrightarrow{d}\;\mathcal N\left(\bm 0 ,\;\bm{I}_m\right)
\]
\end{proof}

\begin{figure}[t!]
\centering
\includegraphics[width=1.0\linewidth]{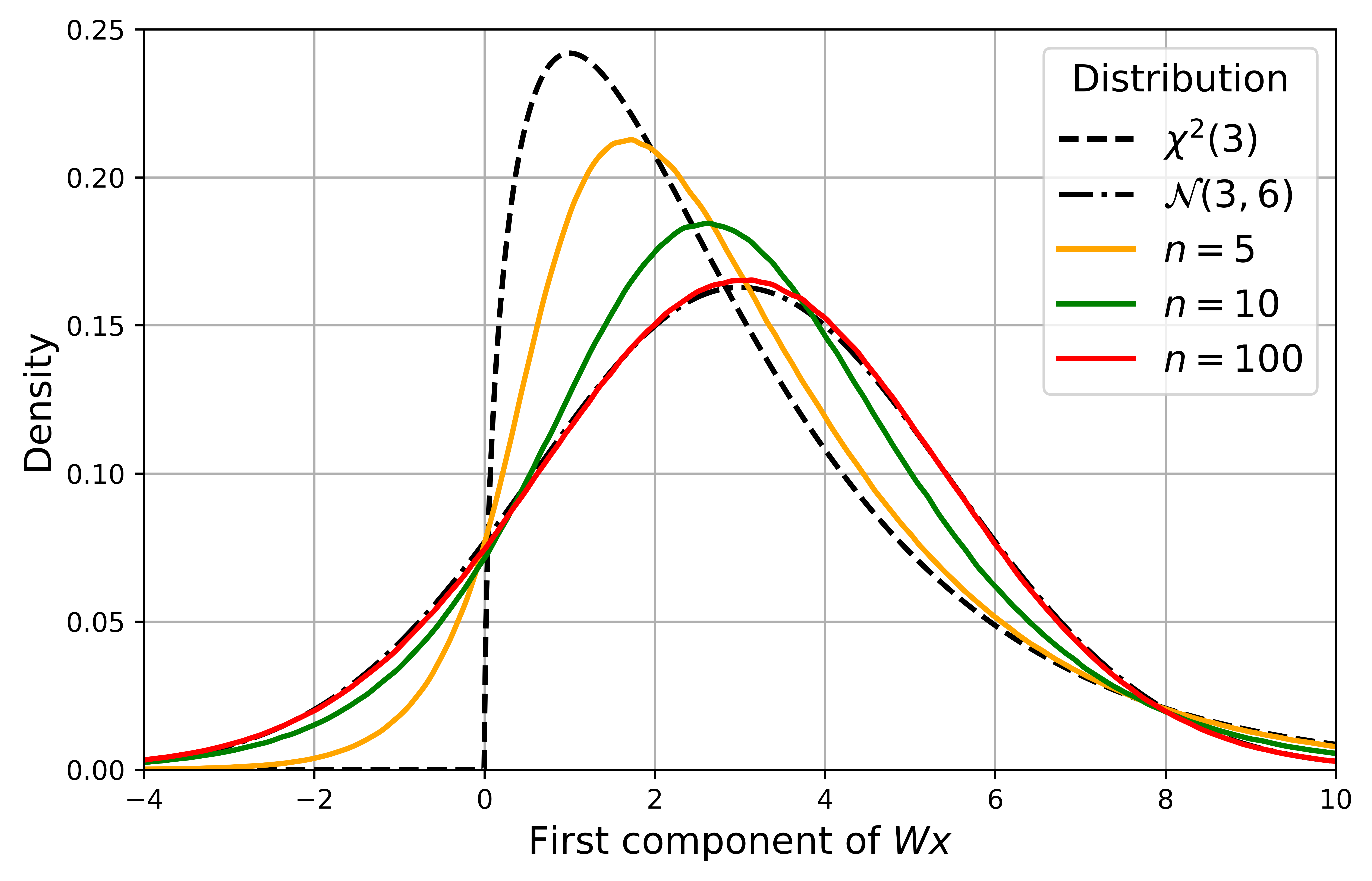}
\caption{Convergence of \( \bm{W} \bm{x} \) to normal distribution as dimension \( n \) increases.}
\label{fig:kde}
\end{figure}

To empirically validate this result, we consider a non-Gaussian input distribution.
Fig.~\ref{fig:kde} shows the empirical distribution of the first component of \( \bm{W} \bm{x} \), where \( \bm{W} \in \widetilde{\mathcal{O}}_{m,n} \) and each entry of \( \bm{x} \in \mathbb{R}^n \) is independently drawn from a chi-squared distribution with three degrees of freedom, \( \chi^2(3) \), for \( m = n \in \{5, 10, 100\} \).
A total of 1,000,000 input vectors were sampled for each value of \( n \) to estimate the output distribution.
The original input distribution is plotted as a dashed line, and a Gaussian distribution with the same mean and variance is plotted as a dash-dotted black line for reference.
As observed in the figure, the transformed distribution becomes increasingly Gaussian as \( n \) increases.

\subsection{Propagation in ReLU Networks}
Deep ReLU networks suffer from variance collapse and dying ReLU problems~\cite{schoenholz2016deep}.
These problems are related to weight initialization and activation functions.
Consider a standard FFNN architecture in which each layer $\ell$ is recursively defined by:
\begin{equation}\label{eq:relu}    
\bm{y}^{(\ell)} = {\bm{W}^{(\ell)}}\bm{x}^{(\ell -1)} + \bm{b}^{(\ell)}, \quad
\bm{x}^{(\ell)} = \text{ReLU}(\bm{y}^{(\ell)}).
\end{equation}
As stated in the preliminary section, we assume a non-increasing architecture where the layer widths \(N_{\ell}\) for $\ell=0,\ldots,L$ satisfy \(N_0 \ge N_1 \ge \cdots \ge N_L\).
Given the central limit behavior established in Theorem~\ref{thm:clt}, it is reasonable to assume that the pre-activation vectors \(\bm{y}^{(\ell)}\) in deep ReLU networks follow approximately Gaussian distributions in high-dimensional settings.
The single-variable rectified Gaussian distribution is defined as \(\max(0, x) \sim \mathcal{N}^{\text R}(\mu , \sigma^2)\) where $x \sim \mathcal{N}(\mu, \sigma^2)$~\cite{beauchamp2018numerical}.
This definition naturally extends to the multivariate.
If $\bm{x} \sim \mathcal{N}(\bm{\mu}, \bm{\Sigma})$ and $\bm{z} = \max(0, \bm{x})$ where the maximum is applied elementwise, then $\bm{z}$  is said to follow a multivariate rectified Gaussian distribution denotes $\bm{z} \sim \mathcal{N}^\mathrm{R}(\bm{\mu}, \bm{\Sigma})$.
Note that $\mathbb{E}[\bm{z}]$ and $\operatorname{Cov}(\bm{z})$ are not $\bm{\mu}$ and $\bm{\Sigma}$, respectively.
Using the result of \cite{beauchamp2018numerical,wright2024analytic}, the following lemma provides the closed-form expressions of $\mathbb{E}[\bm{z}]$ and $\operatorname{Cov}(\bm{z})$ when $\bm{z} \sim \mathcal{N}^R(\mu\bm{1}_n, \sigma^2 \bm{I}_n)$.
\begin{lemma}[\cite{beauchamp2018numerical}]\label{lem:ReLU Gaussian}
    Let $\bm{z} \sim \mathcal{N}^R(\mu\bm{1}_n, \sigma^2 \bm{I}_n)$.
Then
    \begin{align*}
        \mathbb{E}[\bm{z}] & =  \left(\sigma \phi(\alpha)  + \mu \Phi(\alpha)\right)\bm{1}_n\\
        \operatorname{Cov}(\bm{z}) &=  \left((\mu^2 + \sigma^2) \Phi(\alpha) + \mu\sigma\phi(\alpha) - \left(\sigma\phi(\alpha) + \mu \Phi(\alpha) \right)^2\right)\bm{I}_n,
    \end{align*}
    where $\alpha = \frac{\mu}{\sigma}$ and $\phi$ and $\Phi$ are the probability density function and cumulative distribution function of the Gaussian distribution, respectively.
\end{lemma}
However, since \(\lim\limits_{\alpha\to \infty} \Phi(\alpha) = 1\) and \(\lim\limits_{\alpha\to \infty} \phi(\alpha) = 0\), a sufficiently large \(\alpha = \mu / \sigma\) implies that the statistical properties of the ReLU output (e.g. mean and variance) become increasingly close to those of the Gaussian input.
\begin{remark}\label{rmk1}
  Moreover, for each coordinate \(z_i\) of \(\bm{z}\sim\mathcal N(\mu\mathbf1,\sigma^2I)\), \(\mathbb{P}(z_i>0) \;=\;\Phi\bigl(\tfrac\mu\sigma\bigr)\;=\;\Phi(\alpha)\).
\end{remark}
Therefore, as \(\alpha\) becomes large, most of the probability mass of the input lies in the positive region.
In this regime, the ReLU function behaves nearly as the identity map, and its effect on the input distribution becomes negligible.
Consequently, the rectified Gaussian output is nearly indistinguishable from the original Gaussian input in functional terms.
In our setting, the proposed initialization ensures that the pre-activation at layer \(\ell\) follows \(\mathcal{N}\left(\sqrt{\frac{N_{\ell-1}}{N_{\ell}}}\,\mu, \sigma^2\right)\), leading to the rectification parameter:
\[
\alpha_{\ell} = \frac{\mu_{\ell}}{\sigma_{\ell}} = \sqrt{\frac{N_{\ell-1}}{N_{\ell}}}\,\frac{\mu_{\ell-1}}{\sigma_{\ell-1}}.
\]
In deep networks with non-increasing widths (\(N_\ell < N_{\ell-1}\)), the scaling factor \(\sqrt{\frac{N_{\ell-1}}{N_{\ell}}}\) contributes to increasing \(\alpha\).
Furthermore, since ReLU increases the mean and decreases the variance of Gaussian inputs, these effects accumulate across layers.
As a result, \(\alpha\) grows with depth, pushing the activation distribution further into the positive domain.
This behavior alleviates the dying ReLU problem and helps preserve stable signal propagation.
We formalize this observation in the following proposition, which characterizes the depth-wise behavior of activation statistics under our initialization.
\begin{proposition}\label{prop:nomal_new}
    Suppose that the input feature $\bm{x}^{(0)}  = (x^{(0)}_1, \dots, x^{(0)}_{N_0})\in \mathbb{R}^{N_0}$ satisfies the following condition:
    \[
    \mathbb{E}[x^{(0)}_j] = \mu_0, \quad \mathrm{Var}(x^{(0)}_j) = \sigma_0^2, \quad \mathbb{E}[|x^{(0)}_j|^3]<\infty
    \]
    $\text{for all } j = 1, \dots, N_0,$ with input dimension $N_0$ is large.
Assume that, for each layer $\ell$, the weight matrix $\bm{W}^{(\ell)}\in\widetilde{\mathcal O}_{N_{\ell},N_{\ell-1}}$ is a random matrix as defined in Theorem~\ref{lem: big_O}, and that the bias satisfies $\bm{b}^{(\ell)}=\bm{0}$.
Then for each layer \(\ell=1,\ldots,L\), the distribution of the post-activation vector $\bm{x}^{(\ell)}$ is approximately distributed as
    \[
    \bm{x}^{(\ell)} \sim \mathcal{N}^{R}(\mu_\ell \bm{1}_{N_\ell}, \sigma_\ell^2 \bm{I}_{N_\ell}),
    \]
    where $\mu_\ell$ and $\sigma_\ell^2$ denote the mean and variance of the components of $\bm{x}^{(\ell)}$, respectively.
\end{proposition}

\begin{proof}
We proceed by induction on the layer index $\ell$.

\medskip
\noindent First, we prove for $l=1$.
Since \(N_0\) is large, Theorem~\ref{thm:clt} implies that the pre-activation vector \(\bm{y}^{(1)} = \bm{W}^{(1)}\bm{x}^{(0)}\) is approximately Gaussian:
\[
\bm{y}^{(1)} \sim \mathcal{N}\left(\sqrt{\frac{N_0}{N_1}}\,\mu_0\,\bm{1}_{N_1},\, \sigma_0^2\, \bm{I}_{N_1}\right).
\]
Applying the ReLU activation yields \(\bm{x}^{(1)} = \text{ReLU}(\bm{y}^{(1)})\), which is approximately distributed as a rectified Gaussian \(\mathcal{N}^R(\mu_1, \sigma_1^2)\).
This establishes the base case.

\medskip
\noindent Now we prove for $\ell \ge 2$.
Assume the proposition holds for layer $\ell-1$.
We show that $\bm{y}^{(\ell)} = \bm{W}^{(\ell)}\bm{x}^{(\ell-1)}$ remains approximately Gaussian. There are two cases for the input vector $\bm{x}^{(\ell-1)}$.
\medskip
\noindent (i)
If \(N_{\ell-1}\) is large, then by the multivariate central limit theorem, the linear transformation \(\bm{y}^{(\ell)} = \bm{W}^{(\ell)}\bm{x}^{(\ell-1)}\) is approximately Gaussian, even though \(\bm{x}^{(\ell-1)}\) follows a rectified Gaussian distribution.
\medskip
\noindent (ii)
If \(N_{\ell-1}\) is not large, we exploit the fact that the rectification parameter $\alpha_{\ell-1} = \sqrt{\frac{N_{\ell-2}}{N_{\ell-1}}}\frac{\mu_{\ell-2}}{\sigma_{\ell-2}}$ grows with depth.
By the inductive hypothesis, $\bm{x}^{(\ell-1)}$ follows a rectified Gaussian law.
For a sufficiently large $\alpha_{\ell-1}$, the rectified Gaussian distribution of $x_j^{(\ell-1)}$ becomes almost indistinguishable from a true Gaussian distribution, as the mass at zero vanishes.
\[
\bm{x}^{(\ell-1)} \sim \mathcal{N}^R(\mu_{\ell-1}, \sigma_{\ell-1}^2)  \xrightarrow{d} \mathcal{N}(\mu_{\ell-1}, \sigma_{\ell-1}^2) \; \text{as} \; \alpha_{\ell-1} \to \infty.
\]
Therefore, $\bm{y}^{(\ell)} = \bm{W}^{(\ell)}\bm{x}^{(\ell-1)}$ is a linear transformation of an approximately Gaussian vector, which results in $\bm{y}^{(\ell)}$ also being approximately Gaussian.
\medskip
In both cases, the pre-activation $\bm{y}^{(\ell)}$ is approximately Gaussian. Thus, $\bm{x}^{(\ell)} = \text{ReLU}(\bm{y}^{(\ell)})$ is by definition distributed as $\mathcal{N}^R(\mu_\ell, \sigma_\ell^2)$.
By the principle of induction, the proposition holds for all $\ell=1, \dots, L$.
\end{proof}

\subsection{Implications for Signal and Gradient Propagation}

\begin{table*}[ht]
\centering
\scriptsize
\resizebox{0.98\textwidth}{!}{%
\begin{tabular}{lc*{8}{c}}
    \toprule
    Dataset & Depth & Proposed & Lee & ZerO & IDInit & He & Xavier & Orth. & Rand. \\
    \midrule
    \multirow{10}{*}{MNIST}
    & 10  & \textbf{96.96 $\pm$ 0.14} & 95.84 $\pm$ 0.26 & 93.25 $\pm$ 0.65 & 93.50 $\pm$ 0.60 & 96.42 $\pm$ 0.19 & 96.27 $\pm$ 0.15 & 96.46 $\pm$ 0.22 & 95.64 $\pm$ 0.22 \\
    & 20  & \textbf{96.86 $\pm$ 0.25} & 95.79 $\pm$ 0.21 & 91.96 $\pm$ 1.08 & 91.85 $\pm$ 1.24 & 95.65 $\pm$ 0.22 & 94.19 $\pm$ 0.80 & 94.45 $\pm$ 0.59 & 93.04 $\pm$ 0.47 \\
    & 30  & \textbf{96.78 $\pm$ 0.20} & 95.97 $\pm$ 0.22 & 90.70 $\pm$ 2.05 & 91.09 $\pm$ 1.25 & 95.12 $\pm$ 0.24 & 93.41 $\pm$ 0.35 & 92.64 $\pm$ 0.61 & 92.55 $\pm$ 0.63 \\
    & 40  & \textbf{96.69 $\pm$ 0.34} & 96.16 $\pm$ 0.23 & 90.86 $\pm$ 1.28 & 90.83 $\pm$ 1.20 & 94.42 $\pm$ 0.36 & 92.82 $\pm$ 0.46 & 92.83 $\pm$ 0.40 & 91.00 $\pm$ 0.54 \\
    & 50  & \textbf{96.76 $\pm$ 0.16} & 96.44 $\pm$ 0.23 & 90.69 $\pm$ 1.00 & 90.84 $\pm$ 1.06 & 93.94 $\pm$ 0.35 & 92.02 $\pm$ 0.52 & 88.07 $\pm$ 18.07 & 90.57 $\pm$ 0.78 \\
    & 60  & \textbf{96.59 $\pm$ 0.25} & \textbf{96.59 $\pm$ 0.18} & 91.36 $\pm$ 1.05 & 91.11 $\pm$ 0.73 & 93.62 $\pm$ 0.45 & 91.86 $\pm$ 0.45 & 92.00 $\pm$ 0.57 & 89.86 $\pm$ 0.88 \\
    & 70  & \textbf{96.66 $\pm$ 0.20} & 96.61 $\pm$ 0.26 & 91.51 $\pm$ 0.80 & 91.34 $\pm$ 0.74 & 93.32 $\pm$ 0.41 & 83.44 $\pm$ 24.66 & 91.92 $\pm$ 0.59 & 40.28 $\pm$ 36.73 \\
    & 80  & \textbf{96.65 $\pm$ 0.22} & 96.63 $\pm$ 0.21 & 91.36 $\pm$ 1.06 & 91.39 $\pm$ 0.70 & 92.98 $\pm$ 0.61 & 59.52 $\pm$ 40.36 & 91.84 $\pm$ 0.41 & 11.35 $\pm$ 0.00 \\
    & 90  & 96.46 $\pm$ 0.44 & \textbf{96.77 $\pm$ 0.17} & 91.82 $\pm$ 0.67 & 91.72 $\pm$ 0.69 & 92.64 $\pm$ 1.03 & 35.36 $\pm$ 37.63 & 83.58 $\pm$ 24.71 & 11.35 $\pm$ 0.00 \\
    & 100 & 96.57 $\pm$ 0.29 & \textbf{96.89 $\pm$ 0.15} & 91.80 $\pm$ 0.65 & 91.63 $\pm$ 0.91 & 92.45 $\pm$ 0.69 & 35.22 $\pm$ 37.41 & 63.15 $\pm$ 39.01 & 11.35 $\pm$ 0.00 \\
    \midrule
    \multirow{10}{*}{Fashion-MNIST}
    & 10  & \textbf{87.78 $\pm$ 0.24} & 87.02 $\pm$ 0.22 & 87.17 $\pm$ 0.42 & 87.30 $\pm$ 0.30 & 87.35 $\pm$ 0.25 & 87.26 $\pm$ 0.25 & 87.58 $\pm$ 0.35 & 86.22 $\pm$ 0.36 \\
    & 20  & \textbf{87.67 $\pm$ 0.26} & 87.31 $\pm$ 0.25 & 87.25 $\pm$ 0.59 & 87.36 $\pm$ 0.23 & 86.60 $\pm$ 0.33 & 85.30 $\pm$ 0.96 & 85.47 $\pm$ 1.16 & 83.03 $\pm$ 1.18 \\
    & 30  & \textbf{87.57 $\pm$ 0.29} & 87.46 $\pm$ 0.29 & 87.46 $\pm$ 0.36 & 87.44 $\pm$ 0.35 & 86.21 $\pm$ 0.26 & 83.82 $\pm$ 1.04 & 82.76 $\pm$ 0.98 & 50.08 $\pm$ 37.20 \\
    & 40  & \textbf{87.56 $\pm$ 0.29} & 87.52 $\pm$ 0.30 & 87.40 $\pm$ 0.27 & 87.49 $\pm$ 0.43 & 85.80 $\pm$ 0.43 & 79.26 $\pm$ 16.36 & 76.11 $\pm$ 22.63 & 10.00 $\pm$ 0.00 \\
    & 50  & 87.60 $\pm$ 0.35 & \textbf{87.71 $\pm$ 0.29} & 87.41 $\pm$ 0.47 & 87.41 $\pm$ 0.36 & 85.32 $\pm$ 0.46 & 39.17 $\pm$ 36.66 & 21.07 $\pm$ 27.03 & 10.00 $\pm$ 0.00 \\
    & 60  & 87.60 $\pm$ 0.30 & \textbf{87.78 $\pm$ 0.24} & 87.36 $\pm$ 0.46 & 87.52 $\pm$ 0.34 & 84.99 $\pm$ 0.47 & 10.00 $\pm$ 0.00 & 10.00 $\pm$ 0.00 & 10.00 $\pm$ 0.00 \\
    & 70  & 87.56 $\pm$ 0.26 & \textbf{87.82 $\pm$ 0.25} & 87.53 $\pm$ 0.33 & 87.48 $\pm$ 0.52 & 84.59 $\pm$ 0.45 & 10.00 $\pm$ 0.00 & 10.00 $\pm$ 0.00 & 10.00 $\pm$ 0.00 \\
    & 80  & 87.56 $\pm$ 0.36 & \textbf{87.83 $\pm$ 0.24} & 87.40 $\pm$ 0.51 & 87.37 $\pm$ 0.51 & 84.11 $\pm$ 0.64 & 10.00 $\pm$ 0.00 & 10.00 $\pm$ 0.00 & 10.00 $\pm$ 0.00 \\
    & 90  & 87.55 $\pm$ 0.28 & \textbf{87.87 $\pm$ 0.18} & 87.38 $\pm$ 0.52 & 87.44 $\pm$ 0.37 & 83.87 $\pm$ 0.75 & 10.00 $\pm$ 0.00 & 10.00 $\pm$ 0.00 & 10.00 $\pm$ 0.00 \\
    & 100 & 87.52 $\pm$ 0.27 & \textbf{87.88 $\pm$ 0.24} & 87.33 $\pm$ 0.57 & 87.25 $\pm$ 0.68 & 83.71 $\pm$ 0.57 & 10.00 $\pm$ 0.00 & 10.00 $\pm$ 0.00 & 10.00 $\pm$ 0.00 \\
    \bottomrule
\end{tabular}%
}
\caption{Test accuracy (\%) by depth and initialization method. Values are reported as mean $\pm$ standard deviation over twenty random seeds.}
\label{tab:full}
\end{table*}

For a fixed network architecture, ensuring a large rectification parameter, $\alpha_{\ell} = \sqrt{\frac{N_{\ell-1}}{N_{\ell}}}\frac{\mu_{\ell-1}}{\sigma_{\ell-1}}$, involves either increasing the mean $\mu_{\ell-1}$ or decreasing the variance $\sigma_{\ell-1}$.
However, reducing the variance is an undesirable strategy, as it would cause activations to become near-constant.
Therefore, the viable approach is to control the mean.
This mean shift is effective because as $\alpha$ increases, $\Phi(\alpha)$ approaches 1 while the PDF $\phi(\alpha)$ approaches 0.
A key implication is that for inputs with a large $\alpha$, the statistical properties of the signal are almost perfectly preserved after passing through the ReLU activation.
This mean shift mitigates the variance reduction by creating a large rectification parameter, and thus these stable properties can be preserved throughout all subsequent layers.
\medskip

Building on this result, we now examine how our proposed method mitigates common challenges in deep ReLU networks during initialization.
ReLU is known to suffer from three major issues~\cite{he2015delving, hanin2018start}: (a) the \emph{dying ReLU} phenomenon, where neurons become inactive due to persistent negative pre-activations;
(b) the \emph{variance reduction}, where ReLU truncation reduces the dynamic range of signals;
and (c) the \emph{gradient vanishing} problem, especially in deeper networks with improper scaling.
\paragraph{Dying ReLU}
Under the proposed initialization, Proposition~\ref{prop:nomal_new} yields
$\bm{x}^{(\ell)} \sim \mathcal{N}^R(\mu_\ell \mathbf{1},\, \sigma_\ell^2 I)$ with nondecreasing $\mu_\ell$ and nonincreasing $\sigma_\ell$, so the rectification parameter $\alpha_\ell = \mu_\ell/\sigma_\ell$ increases with depth.
By Remark~\ref{rmk1}, $\mathbb{P}(x_j^{(\ell)} > 0)=\Phi(\alpha_\ell)$ and hence the inactivity probability $1-\Phi(\alpha_\ell)$ decays rapidly,  mitigating the dying ReLU phenomenon in deep networks.
\smallskip

\paragraph{Variance reduction}
For $\alpha_{\ell-1}=\mu_{\ell-1}/\sigma_{\ell-1}$, Lemma~\ref{lem:ReLU Gaussian} gives
\begin{align*}
     \frac{\sigma_\ell^2}{\sigma_{\ell-1}^2}
 = &(1+\alpha_{\ell-1}^2)\,\Phi(\alpha_{\ell-1})
   + \alpha_{\ell-1}\,\phi(\alpha_{\ell-1}) \\
   &- \bigl(\phi(\alpha_{\ell-1})+\alpha_{\ell-1}\,\Phi(\alpha_{\ell-1})\bigr)^2 .
\end{align*}

Since $\Phi(a)\to 1$ and $\phi(a)\to 0$ as $a\to\infty$, it follows that $\frac{\sigma_\ell^2}{\sigma_{\ell-1}^2}\to 1$ as $\alpha_{\ell-1}\to\infty$.
Hence, for any $\epsilon\in(0,1)$ there exists $\delta>0$ such that if $\alpha_{\ell-1}\ge \delta$ then $(1-\epsilon)\,\sigma_{\ell-1}^2 \;\le\; \sigma_\ell^2 \;\le\; \sigma_{\ell-1}^2$.
In particular, once $\alpha_{\ell-1}\ge \delta$, the layerwise variance is maintained within a one-sided $(1-\epsilon)$ lower bound, preventing cumulative variance decay and preserving activation dynamic range.
\smallskip

\paragraph{Vanishing gradients}
Let $\bm{D}^{(\ell)}$ be a diagonal matrix whose $i$th diagonal element is $1$ if $x^{(\ell)}_{i}>0$ and $0$ otherwise.
Backpropagation satisfies
\[
\bm{\delta}^{(\ell)}=(\bm{W}^{(\ell+1)})^{T} \bm{D}^{(\ell+1)}\,\bm{\delta}^{(\ell+1)},
\]
where $\bm \delta^{(\ell)}$ is the gradient of the loss with respect to the pre-activation at layer $\ell$.
Under the proposed initialization, $\alpha_\ell=\mu_\ell/\sigma_\ell$ increases with depth, hence
$\mathbb P(x^{(\ell)}_j>0)=\Phi(\alpha_\ell)\to 1$ and $\mathbb E[\bm{D}^{(\ell)}]=\Phi(\alpha_\ell) \bm{I}\to \bm{I}$.
Together with the semi-orthogonality of $\bm{W}^{(\ell)}$ and the variance reduction above, this limits layerwise Jacobian contraction and prevents exponential gradient decay in deep ReLU networks.

\section{Experimental Results}\label{experiment}

\bigskip

This section empirically demonstrates the effectiveness of the proposed initialization on deep ReLU networks.
The proposed initialization is evaluated against several widely adopted initialization methods—Lee, ZerO, IDInit, Orthogonal, Xavier, He, and Random—on the MNIST and Fashion-MNIST datasets. In particular, Lee initialization is included because it has been reported to yield great performance in deep ReLU feed-forward networks. 
ZerO and IDInit are included as recent structured deterministic initialization schemes, whereas Orthogonal, Xavier, He, and Random represent standard choices commonly used in practice.
All experiments employ an FFNN with ReLU activations in all hidden layers, implemented in PyTorch.
The networks are trained using cross-entropy loss and optimized with the Adam algorithm.
The learning rate is scaled with depth as $\text{lr} = 0.0001 / \sqrt{\text{depth}}$, following established practice~\cite{yang2024tensor}.
All models are trained with a batch size of 64 for 100 epochs, and classification accuracy on the test split is reported as the evaluation metric.

\subsection{Effect of Network Depth on Performance}

To examine the behavior of the proposed initialization in deep ReLU networks, we varied the number of hidden layers in a fully connected FFNN with architecture
\[
784 \to \underbrace{64 \to \cdots \to 64}_{D \text{ hidden layers}} \to 10 .
\]
The model was trained separately on the full MNIST and Fashion-MNIST datasets, with a softmax output layer for classification.
Table~\ref{tab:full} reports the resulting test accuracies for different depths and initialization schemes.

For MNIST, the proposed initialization remains highly stable as depth increases, with mean accuracy ranging only from \(96.46\%\) to \(96.96\%\) over \(D=10\) to \(100\).
In particular, no substantial degradation or training collapse is observed even at the deepest settings.
Lee initialization also maintains high accuracy across the same depth range, indicating that both methods remain trainable in regimes where several conventional initialization schemes become unstable.

On Fashion-MNIST, the proposed initialization exhibits similarly stable behavior, maintaining accuracy within a narrow range of \(87.52\%\) to \(87.78\%\) from \(D=10\) to \(100\).
Lee initialization likewise remains competitive throughout the depth range.
These results show that the proposed initialization preserves both trainability and predictive performance as network depth increases, without exhibiting depth-induced collapse.

By contrast, Xavier, Orthogonal, and Random initializations deteriorate rapidly as depth increases, reaching near-random accuracy on Fashion-MNIST by \(D=60\), and showing severe instability or collapse on MNIST at larger depths.
He initialization degrades more gradually, whereas ZerO and IDInit are more stable than Xavier, Orthogonal, and Random, but remain below the proposed initialization on MNIST.
These observations indicate that the proposed initialization provides highly consistent depth-wise performance across both datasets, while avoiding the severe degradation observed for several conventional initialization methods.

\begin{table}[!b]
\centering
\small
\renewcommand{\arraystretch}{1.15}
\resizebox{0.48\textwidth}{!}{%
\begin{tabular}{lcccc}
\toprule
Dataset & Depth & Proposed  & Mean-only & Orthogonal\\
\midrule
\multirow{2}{*}{MNIST}
& 50  & \textbf{96.76 $\pm$ 0.16}  & 95.79 $\pm$ 0.18 & 88.07 $\pm$ 18.07 \\
& 100 & \textbf{96.57 $\pm$ 0.29} & 95.21 $\pm$ 0.40  & 63.15 $\pm$ 39.01 \\
\cmidrule(l){1-5}
\multirow{2}{*}{Fashion-MNIST}
& 50  & \textbf{87.60 $\pm$ 0.35}  & 86.62 $\pm$ 0.22 & 21.07 $\pm$ 27.03 \\
& 100 & \textbf{87.52 $\pm$ 0.27}  & 86.06 $\pm$ 0.29 & 10.00 $\pm$ 0.00 \\
\bottomrule
\end{tabular}
}
\caption{Test accuracy (\%) for the ablation of mean calibration and
semi-orthogonality at depths 50 and 100.
Values are reported as
mean $\pm$ standard deviation over twenty random seeds.}
\label{tab:mean_calibration_ablation}
\end{table}

\begin{table*}[!t]
\centering
\scriptsize
\resizebox{0.98\textwidth}{!}{%
\begin{tabular}{llc*{8}{c}}
\toprule
Dataset & Activation & Depth
& Proposed & Lee & ZerO & IDInit & He & Xavier & Orth. & Rand. \\
\midrule

\multirow{8}{*}{MNIST}
& \multirow{2}{*}{LeakyReLU}
& 50
& 97.30 $\pm$ 0.29
& \textbf{97.70 $\pm$ 0.15}
& 97.26 $\pm$ 0.20
& 97.24 $\pm$ 0.23
& 96.73 $\pm$ 0.26
& 76.16 $\pm$ 33.57
& 72.53 $\pm$ 32.62
& 11.35 $\pm$ 0.00 \\
&
& 100
& 94.56 $\pm$ 5.51
& \textbf{97.66 $\pm$ 0.18}
& 97.19 $\pm$ 0.19
& 97.17 $\pm$ 0.20
& 91.16 $\pm$ 5.58
& 11.35 $\pm$ 0.00
& 11.35 $\pm$ 0.00
& 11.35 $\pm$ 0.00 \\
\cmidrule(lr){2-11}

& \multirow{2}{*}{PReLU}
& 50
& 97.36 $\pm$ 0.20
& \textbf{97.75 $\pm$ 0.22}
& 97.34 $\pm$ 0.22
& 97.31 $\pm$ 0.21
& 96.93 $\pm$ 0.26
& 95.28 $\pm$ 1.80
& 95.68 $\pm$ 3.92
& 35.16 $\pm$ 37.37 \\
&
& 100
& 96.53 $\pm$ 1.42
& \textbf{97.73 $\pm$ 0.18}
& 97.22 $\pm$ 0.21
& 97.24 $\pm$ 0.15
& 91.17 $\pm$ 5.98
& 11.35 $\pm$ 0.00
& 11.35 $\pm$ 0.00
& 11.35 $\pm$ 0.00 \\
\cmidrule(lr){2-11}

& \multirow{2}{*}{ELU}
& 50
& 97.63 $\pm$ 0.18
& 97.38 $\pm$ 0.17
& 97.37 $\pm$ 0.19
& 97.47 $\pm$ 0.14
& 97.18 $\pm$ 0.24
& 97.46 $\pm$ 0.18
& \textbf{97.73 $\pm$ 0.21}
& 95.97 $\pm$ 0.84 \\
&
& 100
& \textbf{97.56 $\pm$ 0.21}
& 97.38 $\pm$ 0.15
& 97.41 $\pm$ 0.29
& 97.43 $\pm$ 0.15
& 91.58 $\pm$ 4.24
& 65.86 $\pm$ 30.12
& 97.10 $\pm$ 1.64
& 11.35 $\pm$ 0.00 \\
\cmidrule(lr){2-11}

& \multirow{2}{*}{SELU}
& 50
& 97.31 $\pm$ 0.21
& \textbf{97.33 $\pm$ 0.18}
& 96.95 $\pm$ 0.24
& 96.83 $\pm$ 0.41
& 96.84 $\pm$ 0.35
& 97.12 $\pm$ 0.29
& 97.32 $\pm$ 0.26
& 97.21 $\pm$ 0.16 \\
&
& 100
& 97.14 $\pm$ 0.26
& \textbf{97.24 $\pm$ 0.29}
& 95.71 $\pm$ 0.35
& 95.69 $\pm$ 0.28
& 91.08 $\pm$ 3.61
& 94.97 $\pm$ 2.51
& 96.98 $\pm$ 0.40
& 94.01 $\pm$ 2.64 \\

\midrule

\multirow{8}{*}{Fashion-MNIST}
& \multirow{2}{*}{LeakyReLU}
& 50
& 87.62 $\pm$ 0.29
& \textbf{87.69 $\pm$ 0.28}
& 87.59 $\pm$ 0.26
& 87.43 $\pm$ 0.36
& 85.36 $\pm$ 0.38
& 38.70 $\pm$ 36.09
& 24.53 $\pm$ 29.82
& 10.00 $\pm$ 0.00 \\
&
& 100
& 87.29 $\pm$ 0.49
& \textbf{88.05 $\pm$ 0.22}
& 87.57 $\pm$ 0.28
& 87.46 $\pm$ 0.33
& 83.45 $\pm$ 0.91
& 10.00 $\pm$ 0.00
& 10.00 $\pm$ 0.00
& 10.00 $\pm$ 0.00 \\
\cmidrule(lr){2-11}

& \multirow{2}{*}{PReLU}
& 50
& 87.61 $\pm$ 0.28
& 87.63 $\pm$ 0.28
& \textbf{87.72 $\pm$ 0.37}
& 87.62 $\pm$ 0.29
& 85.76 $\pm$ 0.30
& 83.60 $\pm$ 0.98
& 82.97 $\pm$ 2.06
& 10.00 $\pm$ 0.00 \\
&
& 100
& 87.36 $\pm$ 0.32
& \textbf{88.04 $\pm$ 0.23}
& 87.71 $\pm$ 0.33
& 87.52 $\pm$ 0.46
& 84.85 $\pm$ 0.43
& 10.00 $\pm$ 0.00
& 10.00 $\pm$ 0.00
& 10.00 $\pm$ 0.00 \\
\cmidrule(lr){2-11}

& \multirow{2}{*}{ELU}
& 50
& \textbf{87.82 $\pm$ 0.21}
& 87.61 $\pm$ 0.25
& 87.82 $\pm$ 0.24
& 87.80 $\pm$ 0.28
& 86.15 $\pm$ 0.44
& 86.79 $\pm$ 0.32
& 87.67 $\pm$ 0.19
& 82.00 $\pm$ 2.12 \\
&
& 100
& 87.68 $\pm$ 0.33
& 87.89 $\pm$ 0.20
& \textbf{88.06 $\pm$ 0.26}
& 88.01 $\pm$ 0.25
& 85.38 $\pm$ 0.82
& 82.90 $\pm$ 2.07
& 86.97 $\pm$ 0.32
& 10.00 $\pm$ 0.00 \\
\cmidrule(lr){2-11}

& \multirow{2}{*}{SELU}
& 50
& 87.42 $\pm$ 0.23
& \textbf{87.81 $\pm$ 0.39}
& 87.33 $\pm$ 0.60
& 87.33 $\pm$ 0.27
& 85.24 $\pm$ 0.33
& 86.17 $\pm$ 0.32
& 85.48 $\pm$ 0.46
& 86.08 $\pm$ 0.46 \\
&
& 100
& \textbf{86.38 $\pm$ 0.36}
& 86.21 $\pm$ 0.94
& 84.70 $\pm$ 0.63
& 84.04 $\pm$ 0.70
& 83.62 $\pm$ 0.78
& 85.55 $\pm$ 0.31
& 83.44 $\pm$ 0.63
& 85.50 $\pm$ 0.29 \\

\bottomrule
\end{tabular}%
}
\caption{Accuracy (\%) of different initialization methods and various activation functions on MNIST and Fashion-MNIST at depths $50$ and $100$. Values are reported as mean $\pm$ standard deviation over twenty random seeds.}
\label{tab:ffn_relu_mean_std}
\end{table*}

\subsection{Ablation of Mean Calibration and Semi-Orthogonality}
\label{sec:mean_calibration_ablation}

To disentangle the respective roles of mean calibration and
semi-orthogonality, we compare the proposed initialization with the
standard Orthogonal initialization and an additional Mean-only
control.
Orthogonal initialization, included in the main
experiments, preserves orthogonality but does not explicitly align the
weight matrix with the all-ones direction.
The Mean-only control is constructed to preserve the same mean
direction as the proposed initialization while removing the
semi-orthogonality constraint.
More precisely, let
$\bm{G}=[G_{ij}]\in\mathbb{R}^{m\times n}$, where
$G_{ij}\sim\mathcal{N}(0,1/n)$, and define
\[
\bm{G}_{\perp}
=
\bm{G}
-
(\bm{G}\bm{\xi}_n)\bm{\xi}_n^T.
\]
The Mean-only initialization is then constructed as
\[
\bm{W}_{\mathrm{mean}}
=
\sqrt{\frac{n}{n-1}}\bm{G}_{\perp}
+
\bm{\xi}_m\bm{\xi}_n^T.
\]
Since $\bm{G}_{\perp}\bm{\xi}_n=\bm{0}$, it follows that
\[
\bm{W}_{\mathrm{mean}}\bm{\xi}_n
=
\bm{\xi}_m.
\]
Hence, for an input satisfying
$\mathbb{E}[\bm{x}]=\mu\bm{1}_n$,
\[
\mathbb{E}[\bm{W}_{\mathrm{mean}}\bm{x}]
=
\mu\sqrt{\frac{n}{m}}\bm{1}_m,
\]
which is the same mean-propagation property as that of the proposed
initialization in Proposition~\ref{pro:clt}.
However
$\bm{W}_{\mathrm{mean}}$ is not constrained to satisfy
$\bm{W}_{\mathrm{mean}}\bm{W}_{\mathrm{mean}}^T=\bm{I}_m$.
Thus, the Mean-only and Orthogonal initializations provide
complementary controls for separating the effects of mean calibration
and semi-orthogonality.

\begin{table*}[!t]
\centering
\small
\resizebox{0.98\textwidth}{!}{
\begin{tabular}{lccccccccc}
\toprule
{Dataset ($k$)} & {Depth} & {Proposed} & {Lee} & {ZerO} & {IDInit} & {He} & {Xavier} & {Orthogonal} & {Random} \\
\midrule
\multirow{3}{*}{MNIST (1)} 
& 10  & \textbf{43.33 $\pm$ 2.89} & 34.54 $\pm$ 4.53 & 9.94 $\pm$ 1.09 & 9.94 $\pm$ 1.10 & 27.42 $\pm$ 2.93 & 28.09 $\pm$ 4.50 & 32.54 $\pm$ 3.58 & 23.79 $\pm$ 4.87 \\
& 50  & \textbf{39.95 $\pm$ 4.70} & 29.24 $\pm$ 4.64 & 9.94 $\pm$ 1.08 & 9.93 $\pm$ 1.01 & 13.80 $\pm$ 2.62 & 10.78 $\pm$ 1.41 & 10.00 $\pm$ 1.45 & 10.08 $\pm$ 0.63 \\
& 100 & \textbf{37.33 $\pm$ 4.32} & 27.43 $\pm$ 4.86 & 9.93 $\pm$ 1.03 & 9.94 $\pm$ 1.09 & 12.85 $\pm$ 2.92 & 9.93 $\pm$ 0.44 & 9.92 $\pm$ 0.40 & 9.98 $\pm$ 0.69 \\
\cmidrule(l){1-10}
\multirow{3}{*}{MNIST (2)} 
& 10  & \textbf{52.18 $\pm$ 4.02} & 41.35 $\pm$ 4.49 & 10.51 $\pm$ 1.81 & 10.77 $\pm$ 2.39 & 34.35 $\pm$ 5.01 & 35.85 $\pm$ 5.66 & 42.22 $\pm$ 4.20 & 28.67 $\pm$ 5.05 \\
& 50  & \textbf{47.83 $\pm$ 3.61} & 34.13 $\pm$ 4.07 & 10.57 $\pm$ 1.92 & 10.57 $\pm$ 1.93 & 15.52 $\pm$ 2.90 & 10.64 $\pm$ 1.52 & 9.91 $\pm$ 0.53 & 9.93 $\pm$ 0.41 \\
& 100 & \textbf{45.75 $\pm$ 4.27} & 34.17 $\pm$ 4.18 & 10.30 $\pm$ 1.70 & 10.31 $\pm$ 1.71 & 13.57 $\pm$ 2.47 & 9.86 $\pm$ 0.67 & 10.17 $\pm$ 0.59 & 10.10 $\pm$ 0.81 \\
\cmidrule(l){1-10}
\multirow{3}{*}{MNIST (4)} 
& 10  & \textbf{63.95 $\pm$ 2.31} & 50.04 $\pm$ 3.45 & 10.85 $\pm$ 2.22 & 10.78 $\pm$ 2.07 & 43.67 $\pm$ 4.90 & 44.73 $\pm$ 4.84 & 52.28 $\pm$ 3.52 & 35.73 $\pm$ 5.96 \\
& 50  & \textbf{58.71 $\pm$ 2.52} & 43.98 $\pm$ 5.84 & 10.76 $\pm$ 2.06 & 10.80 $\pm$ 2.18 & 16.41 $\pm$ 2.95 & 11.62 $\pm$ 2.11 & 9.96 $\pm$ 0.47 & 9.98 $\pm$ 0.65 \\
& 100 & \textbf{58.03 $\pm$ 2.83} & 41.67 $\pm$ 4.83 & 10.78 $\pm$ 2.27 & 10.88 $\pm$ 2.49 & 13.90 $\pm$ 2.17 & 10.08 $\pm$ 0.45 & 9.76 $\pm$ 0.49 & 10.15 $\pm$ 0.70 \\
\cmidrule(l){1-10}
\multirow{3}{*}{MNIST (8)} 
& 10  & \textbf{71.86 $\pm$ 2.17} & 58.85 $\pm$ 2.94 & 12.55 $\pm$ 3.08 & 12.57 $\pm$ 3.14 & 53.92 $\pm$ 4.76 & 52.90 $\pm$ 3.76 & 61.97 $\pm$ 3.80 & 44.67 $\pm$ 6.13 \\
& 50  & \textbf{68.46 $\pm$ 2.78} & 53.01 $\pm$ 5.08 & 13.04 $\pm$ 3.64 & 12.95 $\pm$ 3.41 & 17.96 $\pm$ 3.45 & 10.94 $\pm$ 2.46 & 9.99 $\pm$ 0.67 & 9.84 $\pm$ 0.50 \\
& 100 & \textbf{67.11 $\pm$ 2.38} & 53.19 $\pm$ 4.73 & 13.14 $\pm$ 3.63 & 13.24 $\pm$ 3.65 & 14.45 $\pm$ 2.44 & 9.89 $\pm$ 0.60 & 9.93 $\pm$ 0.44 & 10.05 $\pm$ 0.53 \\
\midrule
\multirow{3}{*}{Fashion-MNIST (1)} 
& 10  & \textbf{47.95 $\pm$ 4.04} & 43.98 $\pm$ 5.96 & 27.05 $\pm$ 6.66 & 27.11 $\pm$ 6.74 & 37.19 $\pm$ 5.39 & 38.98 $\pm$ 4.77 & 42.49 $\pm$ 5.99 & 31.05 $\pm$ 5.13 \\
& 50  & \textbf{48.23 $\pm$ 4.04} & 37.90 $\pm$ 4.36 & 20.20 $\pm$ 5.06 & 21.28 $\pm$ 5.44 & 19.69 $\pm$ 3.81 & 11.20 $\pm$ 2.68 & 10.68 $\pm$ 2.21 & 10.00 $\pm$ 0.00 \\
& 100 & \textbf{46.04 $\pm$ 4.12} & 38.00 $\pm$ 4.15 & 18.71 $\pm$ 5.52 & 19.10 $\pm$ 5.68 & 14.32 $\pm$ 2.83 & 10.00 $\pm$ 0.00 & 10.00 $\pm$ 0.00 & 10.00 $\pm$ 0.00 \\
\cmidrule(l){1-10}
\multirow{3}{*}{Fashion-MNIST (2)} 
& 10  & \textbf{57.09 $\pm$ 3.18} & 50.80 $\pm$ 4.75 & 34.93 $\pm$ 7.19 & 35.08 $\pm$ 7.71 & 46.41 $\pm$ 4.31 & 46.39 $\pm$ 4.46 & 51.37 $\pm$ 3.41 & 40.98 $\pm$ 7.42 \\
& 50  & \textbf{55.77 $\pm$ 4.02} & 47.03 $\pm$ 4.61 & 25.79 $\pm$ 7.52 & 26.77 $\pm$ 7.65 & 22.64 $\pm$ 5.42 & 11.95 $\pm$ 2.44 & 10.00 $\pm$ 0.00 & 10.00 $\pm$ 0.00 \\
& 100 & \textbf{55.07 $\pm$ 4.20} & 45.97 $\pm$ 3.53 & 23.37 $\pm$ 6.00 & 24.32 $\pm$ 6.14 & 16.90 $\pm$ 3.82 & 10.00 $\pm$ 0.00 & 10.00 $\pm$ 0.00 & 10.00 $\pm$ 0.00 \\
\cmidrule(l){1-10}
\multirow{3}{*}{Fashion-MNIST (4)} 
& 10  & \textbf{62.55 $\pm$ 2.91} & 56.54 $\pm$ 3.31 & 40.12 $\pm$ 7.13 & 41.57 $\pm$ 6.82 & 52.76 $\pm$ 3.85 & 53.66 $\pm$ 3.70 & 57.30 $\pm$ 3.42 & 49.86 $\pm$ 4.25 \\
& 50  & \textbf{62.23 $\pm$ 3.06} & 54.03 $\pm$ 3.99 & 30.39 $\pm$ 7.47 & 31.90 $\pm$ 7.10 & 26.24 $\pm$ 5.08 & 13.96 $\pm$ 3.70 & 10.38 $\pm$ 1.48 & 10.00 $\pm$ 0.00 \\
& 100 & \textbf{61.78 $\pm$ 2.85} & 53.46 $\pm$ 3.87 & 27.11 $\pm$ 6.46 & 28.56 $\pm$ 6.85 & 18.52 $\pm$ 3.82 & 10.00 $\pm$ 0.00 & 10.00 $\pm$ 0.00 & 10.00 $\pm$ 0.00 \\
\cmidrule(l){1-10}
\multirow{3}{*}{Fashion-MNIST (8)} 
& 10  & \textbf{67.51 $\pm$ 1.46} & 62.18 $\pm$ 2.24 & 49.19 $\pm$ 6.62 & 49.19 $\pm$ 6.00 & 60.21 $\pm$ 2.45 & 58.90 $\pm$ 2.69 & 61.52 $\pm$ 1.70 & 54.42 $\pm$ 3.70 \\
& 50  & \textbf{66.54 $\pm$ 1.30} & 60.74 $\pm$ 2.79 & 38.07 $\pm$ 6.30 & 38.19 $\pm$ 6.66 & 29.28 $\pm$ 5.38 & 15.26 $\pm$ 3.66 & 10.00 $\pm$ 0.00 & 10.00 $\pm$ 0.00 \\
& 100 & \textbf{65.76 $\pm$ 1.93} & 61.13 $\pm$ 2.20 & 33.95 $\pm$ 7.61 & 33.22 $\pm$ 6.53 & 20.63 $\pm$ 3.83 & 10.00 $\pm$ 0.00 & 10.00 $\pm$ 0.00 & 10.00 $\pm$ 0.00 \\
\bottomrule
\end{tabular}
}
\caption{Few-shot classification accuracy (\%) for various initialization methods and network depths.
The number following each dataset name indicates the number of shots ($k$). Values are reported as mean $\pm$ standard deviation.
The highest mean accuracy in each row is shown in bold.}

\label{tab:few_shot}
\end{table*}
\begin{figure*}[!t]
  \centering
  \begin{subfigure}[b]{0.45\textwidth}
    \includegraphics[width=\linewidth]{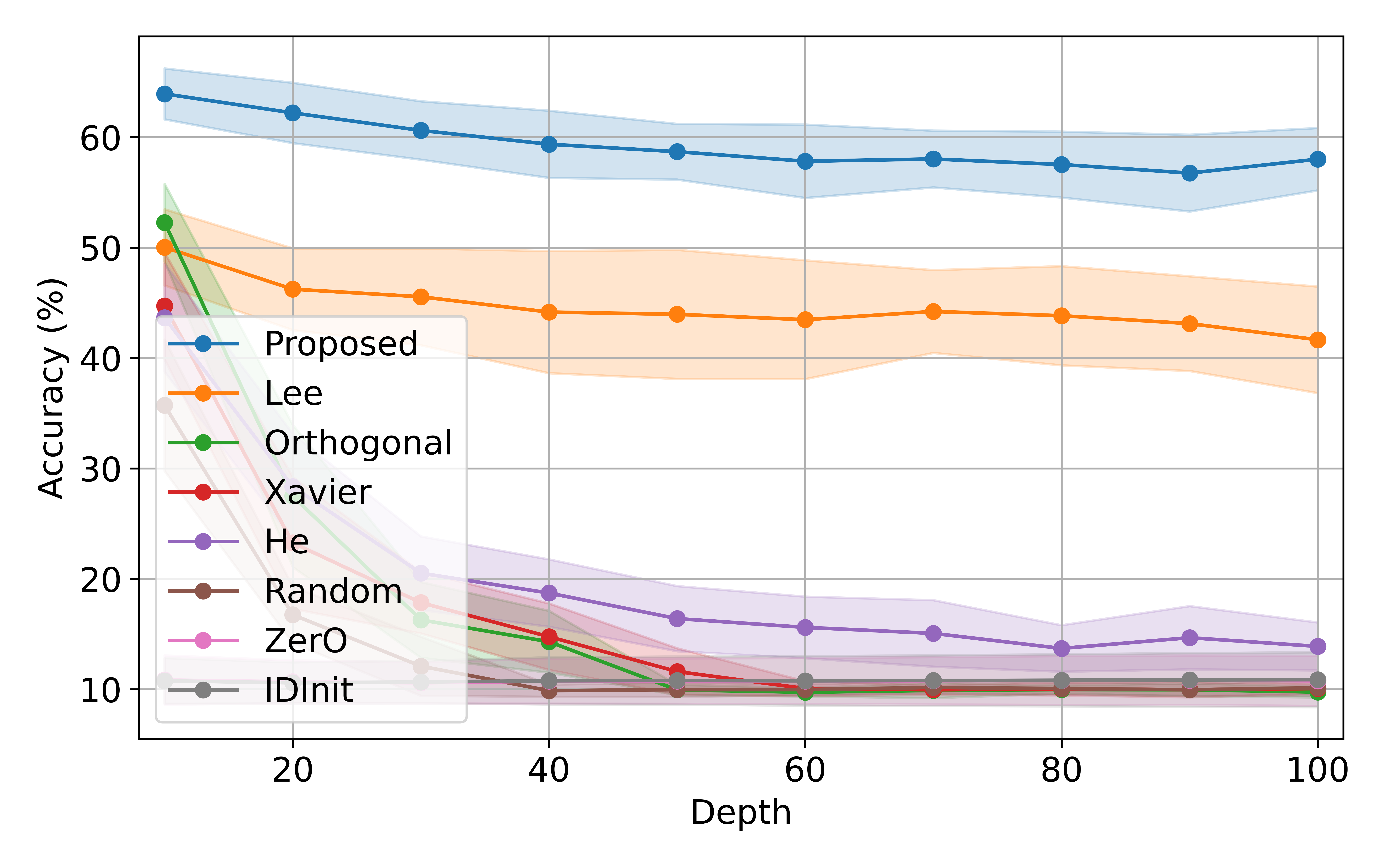}
    \caption{MNIST 4-shot}
    \label{fig:mnist-4s}
  \end{subfigure}
  \begin{subfigure}[b]{0.45\textwidth}
    \includegraphics[width=\linewidth]{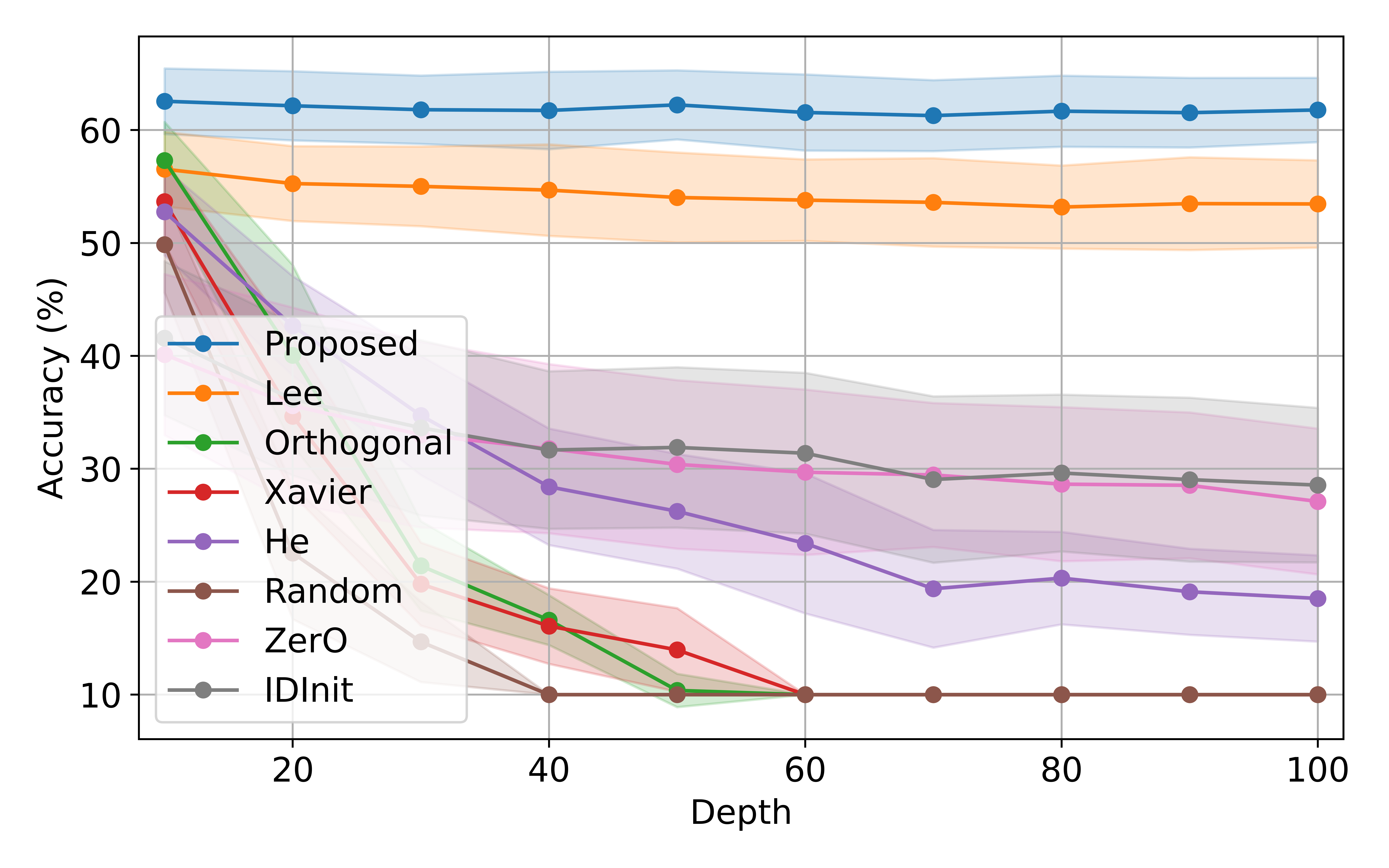}
    \caption{Fashion-MNIST 4-shot}
    \label{fig:fmnist-4s}
  \end{subfigure}
  \caption{4-shot classification accuracy for various initialization methods and network depths.}
  \label{fig:deep‐shot‐grid}
\end{figure*}

Table~\ref{tab:mean_calibration_ablation} shows that mean calibration plays a primary role in preventing depth-induced training collapse. The Mean-only initialization maintains stable performance from $D=50$ to \(D=100\), whereas the standard Orthogonal initialization shows substantial performance degradation with increasing depth. The proposed initialization consistently achieves higher mean accuracy than the Mean-only control across all settings, indicating that semi-orthogonality provides an additional benefit beyond mean calibration alone. These results suggest that mean calibration plays a primary role in maintaining trainability, while semi-orthogonality provides an additional improvement in accuracy and depth-wise stability.

\subsection{Effectiveness across ReLU Family Activation}
To test the robustness of our initialization, we examine its interaction with other popular activation functions from the ReLU family.
To examine the interaction between activation variants and initialization schemes, we replaced the ReLU activation in each hidden layer with LeakyReLU~\cite{maas2013rectifier}, PReLU~\cite{he2015delving}, ELU~\cite{clevert2015fast}, and SELU~\cite{klambauer2017self}, respectively.
The network architecture and training protocol were identical to those in the previous experiments, and the same set of initialization methods was evaluated.
We tested depths of $D=50$ and $D=100$ to highlight failures due to poor initialization, as confirmed by the results in Table~\ref{tab:full}.
The proposed initialization achieves high accuracy on MNIST and Fashion-MNIST across LeakyReLU, PReLU, ELU, and SELU, with a drop of less than one percentage point in most cases when increasing the depth from \(D=50\) to \(D=100\).
In particular, its accuracy remains nearly unchanged under ELU on both datasets, and it achieves the highest accuracy on MNIST at \(D=100\).
These results indicate that the proposed initialization generally maintains stable performance across different ReLU-family activations.
Lee initialization also performs competitively and maintains high accuracy at \(D=100\).
In particular, it outperforms the proposed method under LeakyReLU and PReLU at this depth.
ZerO and IDInit also remain relatively stable, especially under LeakyReLU and PReLU, and achieve performance comparable to the proposed method in several settings.
He initialization consistently supports training across depths and activation types but generally attains lower accuracy than the proposed method, particularly on MNIST at \(D=100\).
By contrast, Xavier, Orthogonal, and Random initializations frequently exhibit severe degradation at \(D=100\), especially under LeakyReLU and PReLU.
Their performance improves under some ELU and SELU settings, indicating that the interaction between initialization and activation choice remains important at large depth.

\subsection{Few-Shot Learning}
To evaluate the proposed initialization method in data-scarce scenarios, we conducted a series of few-shot learning experiments.
These scenarios are particularly challenging because the model must learn to generalize from a tiny number of training examples per class, making the initial state of the network weights a critical factor for successful training and convergence.
The experiments were conducted with the number of training data $k = 1, 2, 4, 8$ using the same network architecture and comparison methods as in the previous section.
Each configuration was repeated 50 times to obtain reliable estimates
of the mean performance and its variability.
As shown in Table~\ref{tab:few_shot}, the proposed initialization shows a clear advantage in few-shot learning across both datasets.
For all numbers of shots ($k=1,2,4,8$) and network depths ($D=10,50,100$), it consistently achieves the highest mean accuracy among the compared initialization methods.

Most conventional initialization methods, including Xavier, He, Orthogonal, and Random, deteriorate substantially as the network becomes deeper, while ZerO and IDInit also show markedly lower accuracy in the few-shot regime.
Lee initialization remains the most competitive baseline and maintains reasonable performance at large depths; however, the proposed method consistently retains higher accuracy, with smaller variability in most of the tested few-shot settings.

Fig.~\ref{fig:deep‐shot‐grid} provides a representation of these trends for the $4$-shot learning.
The shaded area represents the standard deviation of the accuracy of each method.
The relatively flat curve of the proposed initialization further demonstrates its robustness to increasing depth under limited training data.
Overall, these results indicate that few-shot learning is a regime in which the proposed initialization provides a particularly clear and consistent advantage.

\subsection{Tabular Data}

\begin{table}[!ht]
\centering
\resizebox{0.48\textwidth}{!}{%
\small
\renewcommand{\arraystretch}{1.15}
\begin{tabular}{l r r r l l}
\toprule
Dataset & Samples & Features $N_0$ & Output $N_D$ & Task & Architecture \\
\midrule
Adult                    & 48{,}842 & 14 & 2 & Classification & \(\,14 \rightarrow 8 
\rightarrow \cdots \rightarrow 8 \rightarrow 2\,\) \\
Cancer     & 569      & 30 & 2 & Classification & \(\,30 \rightarrow 16 \rightarrow \cdots \rightarrow 16 \rightarrow 2\,\) \\
Pima          & 768      & 8  & 2 & Classification & \(\,8 \rightarrow 4 \rightarrow \cdots \rightarrow 4 \rightarrow 2\,\) \\
Ionosphere               & 351      & 34 & 2 & Classification & \(\,34 \rightarrow 16 
\rightarrow \cdots \rightarrow 16 \rightarrow 2\,\) \\
Wine                     & 178      & 13 & 3 & Classification & \(\,13 \rightarrow 8  \rightarrow \cdots \rightarrow 8 \rightarrow 3\,\) \\
Diabetes       & 442      & 10 & 1 & Regression     & \(\,10\rightarrow 8 \rightarrow \cdots \rightarrow 8 \rightarrow 1\,\) \\
\bottomrule
\end{tabular}
}
\caption{Description of tabular datasets and architectures.}
\label{tab:tabular}
\end{table}

\begin{table*}[!t]
\centering
\small
\resizebox{0.98\textwidth}{!}{%
\begin{tabular}{llccccccccc}
\toprule
{Dataset} & {$\alpha_0$} & {Depth} & {Proposed} & {Lee} & {ZerO} & {IDInit} & {He} & {Xavier} & {Orthogonal} & {Random} \\
\midrule
\multirow{8}{*}{Adult} 
& -2 & 50 
& 76.07 $\pm$ 0.00 
& 85.14 $\pm$ 0.53 
& 85.74 $\pm$ 0.43 
& \textbf{85.86 $\pm$ 0.44} 
& 79.93 $\pm$ 5.29 
& 79.85 $\pm$ 5.17 
& 83.56 $\pm$ 4.21 
& 76.07 $\pm$ 0.00 \\

& & 100 
& 76.07 $\pm$ 0.00 
& 80.06 $\pm$ 4.78 
& \textbf{85.97 $\pm$ 0.46} 
& 85.48 $\pm$ 0.88 
& 79.88 $\pm$ 5.22 
& 76.07 $\pm$ 0.00 
& 76.07 $\pm$ 0.00 
& 76.07 $\pm$ 0.00 \\

\cmidrule(l){2-11}

& 0 & 50 
& 85.40 $\pm$ 0.47 
& 85.90 $\pm$ 0.29 
& \underline{\textbf{86.19 $\pm$ 0.22}} 
& 86.10 $\pm$ 0.24 
& 80.00 $\pm$ 5.38 
& 79.89 $\pm$ 5.23 
& 83.67 $\pm$ 4.27 
& 76.07 $\pm$ 0.00 \\

& & 100 
& 85.41 $\pm$ 0.31 
& 85.85 $\pm$ 0.73
& \underline{\textbf{86.03 $\pm$ 0.54}}
& 86.02 $\pm$ 0.45 
& 79.78 $\pm$ 5.10 
& 76.07 $\pm$ 0.00 
& 77.98 $\pm$ 4.27 
& 76.07 $\pm$ 0.00 \\

\cmidrule(l){2-11}

& 2 & 50 
& 85.31 $\pm$ 0.29 
& 85.54 $\pm$ 0.50 
& \textbf{85.79 $\pm$ 0.47} 
& 85.75 $\pm$ 0.43 
& 81.85 $\pm$ 5.28 
& 76.07 $\pm$ 0.00 
& 83.58 $\pm$ 4.20 
& 76.07 $\pm$ 0.00 \\

& & 100 
& 85.57 $\pm$ 0.36 
& 85.79 $\pm$ 0.53 
& 85.81 $\pm$ 0.39 
& \textbf{85.86 $\pm$ 0.47} 
& 81.49 $\pm$ 4.95 
& 76.07 $\pm$ 0.00 
& 77.84 $\pm$ 3.96 
& 76.07 $\pm$ 0.00 \\

\cmidrule(l){2-11}

& 50 & 50 
& 78.31 $\pm$ 2.99 
& 76.07 $\pm$ 0.00 
& 80.62 $\pm$ 2.08 
& \textbf{82.08 $\pm$ 0.42} 
& 77.30 $\pm$ 2.76 
& 76.07 $\pm$ 0.00 
& 76.07 $\pm$ 0.00 
& 76.07 $\pm$ 0.00 \\

& & 100 
& 76.07 $\pm$ 0.00 
& 76.07 $\pm$ 0.00 
& 75.91 $\pm$ 7.21 
& \textbf{77.09 $\pm$ 10.28} 
& 76.07 $\pm$ 0.00 
& 76.07 $\pm$ 0.00 
& 76.07 $\pm$ 0.00 
& 76.07 $\pm$ 0.00 \\

\midrule

\multirow{8}{*}{Cancer} 
& -2 & 50 
& 91.40 $\pm$ 2.43 
& 91.93 $\pm$ 3.42 
& 91.05 $\pm$ 2.87 
& 90.18 $\pm$ 3.25 
& 89.65 $\pm$ 4.78 
& \textbf{94.56 $\pm$ 2.43} 
& 94.04 $\pm$ 1.14 
& 63.16 $\pm$ 0.00 \\

& & 100 
& 88.07 $\pm$ 2.68 
& 83.16 $\pm$ 5.53 
& \textbf{90.18 $\pm$ 3.25} 
& 85.09 $\pm$ 4.39 
& 86.67 $\pm$ 13.32 
& 63.16 $\pm$ 0.00 
& 63.16 $\pm$ 0.00 
& 63.16 $\pm$ 0.00 \\

\cmidrule(l){2-11}

& 0 & 50 
& \underline{\textbf{95.96 $\pm$ 1.59}} 
& 94.91 $\pm$ 1.57 
& 94.74 $\pm$ 0.88 
& 94.56 $\pm$ 0.39 
& 92.63 $\pm$ 3.70 
& 95.79 $\pm$ 0.73 
& 95.09 $\pm$ 1.82 
& 63.16 $\pm$ 0.00 \\

& & 100 
& \textbf{94.74 $\pm$ 2.70} 
& \textbf{94.74 $\pm$ 1.39} 
& 94.04 $\pm$ 0.73 
& 94.21 $\pm$ 1.33 
& 84.74 $\pm$ 12.29 
& 63.16 $\pm$ 0.00 
& 63.16 $\pm$ 0.00 
& 63.16 $\pm$ 0.00 \\

\cmidrule(l){2-11}

& 2 & 50 
& 95.44 $\pm$ 1.44
& 95.26 $\pm$ 1.33 
& 94.74 $\pm$ 2.06 
& 94.74 $\pm$ 2.06 
& 92.46 $\pm$ 3.95 
& 94.21 $\pm$ 2.02 
& \textbf{95.79 $\pm$ 2.73}
& 63.16 $\pm$ 0.00 \\

& & 100 
& 95.61 $\pm$ 1.39 
& \underline{\textbf{96.49 $\pm$ 1.86}} 
& 95.26 $\pm$ 1.92 
& 95.09 $\pm$ 1.33 
& 92.81 $\pm$ 4.27 
& 63.16 $\pm$ 0.00 
& 63.16 $\pm$ 0.00 
& 63.16 $\pm$ 0.00 \\

\cmidrule(l){2-11}

& 50 & 50 
& 63.68 $\pm$ 12.12 
& 77.19 $\pm$ 7.02 
& \textbf{80.35 $\pm$ 6.64} 
& 79.12 $\pm$ 7.93 
& 63.86 $\pm$ 1.57 
& 63.16 $\pm$ 0.00 
& 63.16 $\pm$ 0.00 
& 63.16 $\pm$ 0.00 \\

& & 100 
& 69.47 $\pm$ 11.86 
& \textbf{84.74 $\pm$ 9.72} 
& 80.00 $\pm$ 4.45 
& 75.79 $\pm$ 5.43 
& 65.96 $\pm$ 6.28 
& 63.16 $\pm$ 0.00 
& 63.16 $\pm$ 0.00 
& 63.16 $\pm$ 0.00 \\

\midrule

\multirow{8}{*}{Pima} 
& -2 & 50 
& 65.32 $\pm$ 0.98 
& 69.35 $\pm$ 1.16 
& \textbf{72.73 $\pm$ 1.59} 
& 72.21 $\pm$ 2.57 
& 64.94 $\pm$ 0.00 
& 64.94 $\pm$ 0.00 
& 64.94 $\pm$ 0.00 
& 64.94 $\pm$ 0.00 \\

& & 100 
& 66.36 $\pm$ 1.07 
& 66.49 $\pm$ 0.87 
& \textbf{72.47 $\pm$ 2.08} 
& 72.08 $\pm$ 2.25 
& 64.94 $\pm$ 0.00 
& 64.94 $\pm$ 0.00 
& 64.94 $\pm$ 0.00 
& 64.94 $\pm$ 0.00 \\

\cmidrule(l){2-11}

& 0 & 50 
& 75.45 $\pm$ 3.48 
& \underline{\textbf{77.27 $\pm$ 3.05}} 
& 75.84 $\pm$ 4.14 
& 75.71 $\pm$ 4.17 
& 64.94 $\pm$ 0.00 
& 64.94 $\pm$ 0.00 
& 64.94 $\pm$ 0.00 
& 64.94 $\pm$ 0.00 \\

& & 100 
& 75.71 $\pm$ 3.86 
& \textbf{76.88 $\pm$ 3.13} 
& 75.58 $\pm$ 4.20 
& 74.81 $\pm$ 3.83 
& 64.94 $\pm$ 0.00 
& 64.94 $\pm$ 0.00 
& 64.94 $\pm$ 0.00 
& 64.94 $\pm$ 0.00 \\

\cmidrule(l){2-11}

& 2 & 50 
& 75.84 $\pm$ 3.74 
& \textbf{77.14 $\pm$ 2.69} 
& 75.45 $\pm$ 2.53 
& 74.68 $\pm$ 2.15 
& 64.94 $\pm$ 0.00 
& 64.94 $\pm$ 0.00 
& 64.94 $\pm$ 0.00 
& 64.94 $\pm$ 0.00 \\

& & 100 
& \underline{\textbf{76.49 $\pm$ 2.73}} 
& \textbf{76.49 $\pm$ 2.57} 
& 75.19 $\pm$ 2.73 
& 74.29 $\pm$ 2.82 
& 64.94 $\pm$ 0.00 
& 64.94 $\pm$ 0.00 
& 64.94 $\pm$ 0.00 
& 64.94 $\pm$ 0.00 \\

\cmidrule(l){2-11}

& 50 & 50 
& 62.73 $\pm$ 1.75 
& 62.47 $\pm$ 2.40 
& \textbf{67.27 $\pm$ 3.54} 
& 67.01 $\pm$ 3.77 
& 64.94 $\pm$ 0.00 
& 64.94 $\pm$ 0.00 
& 64.94 $\pm$ 0.00 
& 64.94 $\pm$ 0.00 \\

& & 100 
& 64.68 $\pm$ 2.74 
& 63.38 $\pm$ 2.32 
& \textbf{69.61 $\pm$ 1.97} 
& 68.18 $\pm$ 1.78 
& 64.94 $\pm$ 0.00 
& 64.94 $\pm$ 0.00 
& 64.94 $\pm$ 0.00 
& 64.94 $\pm$ 0.00 \\

\midrule

\multirow{8}{*}{Ionosphere} 
& -2 & 50 
& 69.01 $\pm$ 1.72 
& 69.30 $\pm$ 2.31 
& 75.21 $\pm$ 2.14 
& 74.65 $\pm$ 2.23 
& 81.13 $\pm$ 4.52 
& \textbf{85.92 $\pm$ 12.20} 
& 85.63 $\pm$ 11.88 
& 64.79 $\pm$ 0.00 \\

& & 100 
& 69.30 $\pm$ 1.84 
& 67.61 $\pm$ 1.99 
& 74.65 $\pm$ 2.23 
& 74.65 $\pm$ 1.99 
& \textbf{81.69 $\pm$ 11.74} 
& 64.79 $\pm$ 0.00 
& 64.79 $\pm$ 0.00 
& 64.79 $\pm$ 0.00 \\

\cmidrule(l){2-11}

& 0 & 50 
& 91.27 $\pm$ 2.09 
& 91.27 $\pm$ 3.65 
& 89.30 $\pm$ 1.61 
& 88.45 $\pm$ 3.65 
& 83.10 $\pm$ 5.89 
& 78.87 $\pm$ 12.91 
& \textbf{91.55 $\pm$ 3.98} 
& 64.79 $\pm$ 0.00 \\

& & 100 
& 86.20 $\pm$ 2.89 
& \textbf{91.83 $\pm$ 1.54} 
& 90.70 $\pm$ 2.56 
& 88.45 $\pm$ 3.91 
& 84.23 $\pm$ 12.09 
& 64.79 $\pm$ 0.00 
& 70.99 $\pm$ 13.86 
& 64.79 $\pm$ 0.00 \\

\cmidrule(l){2-11}

& 2 & 50 
& 93.80 $\pm$ 1.61 
& 93.24 $\pm$ 2.89 
& 92.96 $\pm$ 1.00 
& \underline{\textbf{94.08 $\pm$ 0.63}} 
& 87.61 $\pm$ 4.82 
& 87.61 $\pm$ 3.05 
& 92.68 $\pm$ 5.58 
& 64.79 $\pm$ 0.00 \\

& & 100 
& 92.11 $\pm$ 3.39 
& 86.48 $\pm$ 9.84 
& 93.24 $\pm$ 2.71 
& \underline{\textbf{93.80 $\pm$ 2.36}} 
& 73.24 $\pm$ 8.68 
& 64.79 $\pm$ 0.00 
& 64.79 $\pm$ 0.00 
& 64.79 $\pm$ 0.00 \\

\cmidrule(l){2-11}

& 50 & 50 
& 52.68 $\pm$ 17.41 
& 72.39 $\pm$ 5.04 
& \textbf{75.49 $\pm$ 7.62} 
& 72.11 $\pm$ 2.31 
& 66.76 $\pm$ 4.41 
& 64.79 $\pm$ 0.00 
& 64.79 $\pm$ 0.00 
& 64.79 $\pm$ 0.00 \\

& & 100 
& 70.70 $\pm$ 7.28 
& 69.86 $\pm$ 19.91 
& \textbf{71.83 $\pm$ 17.19} 
& 60.28 $\pm$ 25.06 
& 63.94 $\pm$ 1.89 
& 64.79 $\pm$ 0.00 
& 64.79 $\pm$ 0.00 
& 64.79 $\pm$ 0.00 \\

\midrule

\multirow{8}{*}{Wine} 
& -2 & 50 
& 53.33 $\pm$ 16.83 
& 47.22 $\pm$ 3.93 
& \textbf{53.89 $\pm$ 4.65} 
& 47.22 $\pm$ 5.20 
& \textbf{53.89 $\pm$ 14.91} 
& 38.89 $\pm$ 0.00 
& 38.89 $\pm$ 0.00 
& 38.89 $\pm$ 0.00 \\

& & 100 
& 41.11 $\pm$ 4.56 
& 38.89 $\pm$ 4.39 
& \textbf{42.22 $\pm$ 2.32} 
& 41.67 $\pm$ 2.78 
& 40.56 $\pm$ 3.73 
& 38.89 $\pm$ 0.00 
& 38.89 $\pm$ 0.00 
& 38.89 $\pm$ 0.00 \\

\cmidrule(l){2-11}

& 0 & 50 
& \textbf{91.67 $\pm$ 3.40} 
& 80.56 $\pm$ 0.00 
& 85.56 $\pm$ 3.62 
& 84.44 $\pm$ 4.65 
& 58.89 $\pm$ 13.09 
& 38.89 $\pm$ 0.00 
& 38.89 $\pm$ 0.00 
& 38.89 $\pm$ 0.00 \\

& & 100 
& \textbf{90.00 $\pm$ 7.51} 
& 80.00 $\pm$ 2.32 
& 82.22 $\pm$ 4.21 
& 82.22 $\pm$ 4.21 
& 46.67 $\pm$ 10.65 
& 38.89 $\pm$ 0.00 
& 38.89 $\pm$ 0.00 
& 38.89 $\pm$ 0.00 \\

\cmidrule(l){2-11}

& 2 & 50 
& \underline{\textbf{97.22 $\pm$ 2.78}} 
& 96.67 $\pm$ 2.32 
& 96.11 $\pm$ 1.52 
& 95.00 $\pm$ 1.24 
& 58.33 $\pm$ 11.95 
& 38.89 $\pm$ 0.00 
& 38.89 $\pm$ 0.00 
& 38.89 $\pm$ 0.00 \\

& & 100 
& 92.22 $\pm$ 2.32 
& \underline{\textbf{96.67 $\pm$ 2.32}} 
& 93.89 $\pm$ 1.24 
& 94.44 $\pm$ 0.00 
& 43.33 $\pm$ 9.94 
& 38.89 $\pm$ 0.00 
& 38.89 $\pm$ 0.00 
& 38.89 $\pm$ 0.00 \\

\cmidrule(l){2-11}

& 50 & 50 
& 47.22 $\pm$ 16.55 
& 75.00 $\pm$ 6.80 
& \textbf{76.11 $\pm$ 5.05} 
& 73.89 $\pm$ 8.24 
& 42.78 $\pm$ 8.70 
& 38.89 $\pm$ 0.00 
& 38.89 $\pm$ 0.00 
& 38.89 $\pm$ 0.00 \\

& & 100 
& 52.22 $\pm$ 15.27 
& 72.22 $\pm$ 2.78 
& \textbf{77.78 $\pm$ 7.61} 
& 77.22 $\pm$ 13.52 
& 38.89 $\pm$ 0.00 
& 38.89 $\pm$ 0.00 
& 38.89 $\pm$ 0.00 
& 38.89 $\pm$ 0.00 \\

\midrule

\multirow{8}{*}{Diabetes} 
& -2 & 50 
& 70.59 $\pm$ 6.45 
& 171.59 $\pm$ 7.66 
& \textbf{68.25 $\pm$ 4.39} 
& 69.63 $\pm$ 4.25 
& 71.74 $\pm$ 6.75 
& 92.60 $\pm$ 39.38 
& 89.95 $\pm$ 42.88 
& 93.40 $\pm$ 38.26 \\

& & 100 
& 74.46 $\pm$ 2.45 
& 171.64 $\pm$ 7.66 
& \textbf{69.63 $\pm$ 3.69} 
& 70.27 $\pm$ 3.41 
& 100.69 $\pm$ 42.44 
& 94.90 $\pm$ 43.85 
& 90.03 $\pm$ 41.22 
& 171.64 $\pm$ 7.66 \\

\cmidrule(l){2-11}

& 0 & 50 
& \textbf{58.43 $\pm$ 4.89} 
& 171.59 $\pm$ 7.66 
& 58.62 $\pm$ 3.81 
& 59.29 $\pm$ 3.90 
& 69.34 $\pm$ 7.57 
& 90.87 $\pm$ 40.71 
& 89.31 $\pm$ 43.34 
& 93.40 $\pm$ 38.26 \\

& & 100 
& \textbf{58.39 $\pm$ 4.33} 
& 171.64 $\pm$ 7.66 
& 58.77 $\pm$ 5.05 
& 59.18 $\pm$ 5.20 
& 99.67 $\pm$ 43.24 
& 94.91 $\pm$ 43.85 
& 83.69 $\pm$ 44.90 
& 171.64 $\pm$ 7.66 \\

\cmidrule(l){2-11}

& 2 & 50 
& \underline{\textbf{54.26 $\pm$ 3.09}} 
& 171.59 $\pm$ 7.66 
& 55.86 $\pm$ 4.17 
& 56.06 $\pm$ 4.20 
& 68.43 $\pm$ 5.62 
& 90.65 $\pm$ 40.92 
& 86.92 $\pm$ 44.41 
& 93.40 $\pm$ 38.26 \\

& & 100 
& \underline{\textbf{55.18 $\pm$ 3.93}}
& 171.64 $\pm$ 7.66 
& 56.05 $\pm$ 4.72 
& 56.29 $\pm$ 4.75 
& 98.67 $\pm$ 44.02 
& 94.85 $\pm$ 43.87 
& 89.74 $\pm$ 41.01 
& 171.64 $\pm$ 7.66 \\

\cmidrule(l){2-11}

& 50 & 50 
& 74.42 $\pm$ 4.57 
& 171.69 $\pm$ 7.66 
& 74.74 $\pm$ 3.85 
& 74.63 $\pm$ 3.79 
& \textbf{71.20 $\pm$ 6.36} 
& 93.09 $\pm$ 39.06 
& 94.35 $\pm$ 39.87 
& 93.40 $\pm$ 38.26 \\

& & 100 
& \textbf{74.08 $\pm$ 3.85} 
& 171.67 $\pm$ 7.66 
& 74.67 $\pm$ 4.25 
& 74.50 $\pm$ 4.21 
& 101.06 $\pm$ 41.97 
& 94.91 $\pm$ 43.85 
& 93.35 $\pm$ 39.01 
& 171.64 $\pm$ 7.66 \\

\bottomrule
\end{tabular}%
}
\caption{
Classification accuracy (\%) for the classification datasets and RMSE for Diabetes regression.
Values are reported as mean $\pm$ standard deviation.
Boldface denotes the best performance for each
$(\text{dataset},\alpha_0,\text{depth})$ configuration, while underlined boldface additionally indicates the best $\alpha_0$ for the corresponding initialization method at the same dataset and depth.}
\label{tab:init_tabular}
\end{table*}

To assess generalization beyond image domains, several tabular datasets were used, 
covering both classification (Adult~\cite{adult}, Cancer~\cite{cancer}, Pima~\cite{pima}, Ionosphere~\cite{ionosphere}, Wine~\cite{wine})  
and regression (Diabetes~\cite{diabetes}) tasks (Table~\ref{tab:tabular}).
As tabular inputs require explicit scaling, all features were standardized to unit variance (\(\sigma^2=1\)), and the mean \(\mu\) was shifted to control \( \alpha = \mu/\sigma\).
Experiments used \(\alpha \in \{-2,\,0,\,2,\,50\}\). 
Models were fully connected FFNN(architecture details in Table~\ref{tab:tabular}).
Training settings followed those of the prior experiments, and evaluation used accuracy for classification and root mean squared error (RMSE) for regression.
The tabular results are summarized in Table~\ref{tab:init_tabular}.

Across the tabular classification tasks, the proposed initialization remains competitive across different datasets and depths, with strong performance generally observed for moderate values of $\alpha_0$, particularly $\alpha_0\in\{0,2\}$.
In contrast, negative values of $\alpha_0$ ($\alpha_0=-2$) and excessively large mean shifts ($\alpha_0=50$) tend to degrade performance, indicating that an appropriate level of mean shift is important for stable training.
On the Diabetes regression task, the proposed initialization achieves its lowest RMSE at $\alpha_0=2$ for both depths, whereas Lee initialization, despite its competitive performance on the classification tasks, performs substantially worse in this regression setting.
Overall, these results indicate that the proposed initialization achieves competitive performance across tabular tasks under an appropriate mean shift.

\section{Conclusion}\label{sec:conclusion}

This paper provided a weight initialization for deep ReLU networks by formulating and solving an optimization problem on the Stiefel manifold.
Unlike conventional approaches such as Xavier, He, and orthogonal initialization, the proposed method combines semi-orthogonality with explicit alignment to the all-ones direction.
It maximizes this alignment, thereby directly regulating pre-activation statistics.
We derived closed-form characterizations of the optimal solution set, developed efficient sampling algorithms, and established theoretical guarantees on variance preservation, mean calibration, and asymptotic distributional behavior.
Theoretical analysis demonstrated that the proposed initialization systematically alleviates critical early-stage training issues---most notably the dying ReLU phenomenon, variance decay, and gradient vanishing---by promoting stable signal and gradient propagation across layers.

The empirical results demonstrate that the proposed initialization remains highly stable as network depth increases and avoids the severe degradation or training collapse observed for several conventional initialization methods.
Its advantage is particularly pronounced in few-shot learning, where it consistently achieves the highest mean accuracy across all tested numbers of shots and network depths.
The method also maintains strong performance across different ReLU-family activations and remains competitive on a range of tabular tasks.
The ablation study further shows that mean calibration plays a primary role in maintaining trainability.
The tabular experiments also indicate the importance of selecting an appropriate mean shift.

Taken together, these results highlight the importance of geometrically informed initialization in the design of deep networks.
By exploiting the structure of the Stiefel manifold, our work demonstrates that initialization can play a decisive role in stabilizing training and extending depth scalability.
Future work may extend this framework to convolutional architectures, attention-based models, and adaptive optimization strategies on manifolds, further bridging rigorous mathematical design with practical advances in deep learning.

\section*{Declaration of competing interest}
The authors declare that they have no known competing financial interests or personal relationships that could have appeared to influence the work reported in this paper.

\section*{Acknowledgement}
All authors contribute equally to this paper.
This work of H. Lee, T. Kim, and H. Choi was supported by the National Research Foundation of Korea(NRF) grant funded by the Korea government(MSIT) (No. RS-2022-NR070839 and RS-2024-00342939 and RS-2026-25537933 and RS-2026-25504546). 
T. Kim was supported by the National Research Foundation of Korea(NRF) grant funded by the Korea government(MSIT) (No. RS-2025-25436769).

\bibliographystyle{elsarticle-num}
\bibliography{reference_R2}

\end{document}